\documentclass{article}

\usepackage[nonatbib,final]{neurips_2022}

\usepackage[utf8]{inputenc} %
\usepackage[T1]{fontenc}    %
\usepackage{hyperref}       %
\usepackage{url}            %
\usepackage{booktabs}       %
\usepackage{amsfonts}       %
\usepackage{nicefrac}       %
\usepackage{microtype}      %
\usepackage{xcolor}   %

\usepackage[style=alphabetic,backend=bibtex,maxcitenames=6,maxalphanames=6,maxbibnames=99]{biblatex}
\newcommand{\citep}{\cite}
\newcommand{\citet}{\cite}

\usepackage{listings}
\usepackage{float}
\usepackage{fancyhdr}
\usepackage{graphicx}
\usepackage{amsmath}
\usepackage{amsthm}
\usepackage{amssymb}
\usepackage{listing}
\usepackage{enumitem}
\usepackage{mathtools}
\usepackage{bbm}
\usepackage{wrapfig}
\usepackage{color}

\usepackage[ruled, vlined, linesnumbered]{algorithm2e}
\usepackage{quoting}

\usepackage[parfill]{parskip}

\usepackage{tikz}

\tikzset{
    ->, 
    level distance = 12em,
    minimum size=2em,
    level 1/.style={sibling distance=6em},
    level 2/.style={sibling distance=3em},
    thick/.style = {line width=1.5pt},
    extra thick/.style = {line width=3.5pt},
    red node/.style={shape=circle,draw=red,fill=red!40,thick,inner sep=1.2},
    blue node/.style={shape=circle,draw=blue,fill=blue!40,thick,inner sep=1.2}
}

\tikzstyle{round}=[thick,draw=black,circle]

\theoremstyle{plain}
\newtheorem{theorem}{\textbf{Theorem}}
\newtheorem{lemma}[theorem]{\textbf{Lemma}}
\newtheorem{corollary}[theorem]{Corollary}
\newtheorem{proposition}[theorem]{Proposition}

\newtheorem{definition}{Definition}

\theoremstyle{definition}
\newtheorem{assumption}{Assumption}

\newtheorem{remark}{Remark}

\newcommand{\E}{\mathbb{E}}
\newcommand{\EE}{\mathbb{E}}
\newcommand{\R}{\mathbb{R}}
\newcommand{\RR}{\mathbb{R}}
\newcommand{\I}{\mathbb{I}}
\newcommand{\M}{\mathcal{M}}
\newcommand{\Q}{\mathcal{Q}}
\newcommand{\Qcal}{\mathcal{Q}}
\newcommand{\e}{\epsilon}
\newcommand{\Z}{\mathcal{Z}}
\newcommand{\A}{\mathcal{A}}
\newcommand{\Acal}{\mathcal{A}}
\newcommand{\C}{\mathcal{C}}
\newcommand{\T}{\mathcal{T}}
\newcommand{\Tcal}{\mathcal{T}}
\newcommand{\X}{\mathcal{X}}
\newcommand{\D}{\mathcal{D}}

\newcommand{\Scal}{\mathcal{S}}
\newcommand{\F}{\mathcal{F}}

\newcommand{\U}{\mathcal{U}}
\newcommand{\V}{\mathcal{V}}
\newcommand{\W}{\mathcal{W}}
\newcommand{\Wcal}{\mathcal{W}}

\newcommand{\Pt}{\wt{P}^\pi}
\newcommand{\Ncal}{\mathcal{N}}

\DeclareMathOperator*{\argmin}{arg\,min}
\DeclareMathOperator*{\argmax}{arg\,max}

\newcommand{\wt}{\widetilde}
\newcommand{\wh}{\widehat}
\newcommand{\ol}{\overline}

\newcommand{\pp}{\mathbb{P}}

\newcommand{\B}{\mathcal{B}}

\newcommand{\lb}{\left[} 
\newcommand{\rb}{\right]}
\newcommand{\lv}{\left\vert} 
\newcommand{\rv}{\right\vert}

\newcommand{\lp}{\left(} 
\newcommand{\rp}{\right)}
\newcommand{\llangle}{\Big\langle}
\newcommand{\rrangle}{\Big\rangle}

\newcommand{\ediv}{/} %
\newcommand{\emult}{\circ}

\newcommand{\diag}{\text{Diag}\xspace}

\newcommand{\dr}{{dr}}

\newcommand{\init}{\mu_0}
\newcommand{\initsa}{\mu_0^\pi}

\newcommand{\para}[1]{\textbf{#1}~}

\title{Beyond the Return: Off-policy Function Estimation under %
User-specified Error-measuring Distributions} %

\author{%
  Audrey Huang \\
  Computer Science\\
  University of Illinois at Urbana-Champaign \\
  \texttt{audreyh5@illinois.edu} \And
  Nan Jiang \\
  Computer Science\\
  University of Illinois at Urbana-Champaign \\
  \texttt{nanjiang@illinois.edu}
}

\newcount\Comments  %
\newcommand{\kibitz}[2]{\ifnum\Comments=1{\textcolor{#1}{\textsf{\footnotesize #2}}}\fi}

\begin{document}

\maketitle

\begin{abstract}
Off-policy evaluation often refers to two related tasks: estimating the expected return of a policy and estimating its value function (or other functions of interest, such as density ratios). While recent works on marginalized importance sampling (MIS) show that the former can enjoy provable guarantees under realizable function approximation, the latter is only known to be feasible under much stronger assumptions  such as prohibitively expressive discriminators. 
In this work, we provide guarantees for off-policy function estimation under  only realizability, by imposing proper regularization on the MIS objectives. Compared to commonly used regularization in MIS, 
our regularizer is much more flexible and can account for an arbitrary user-specified distribution, under which the learned function will be close to the groundtruth. We  provide  \textit{exact} characterization of the optimal dual solution that needs to be realized by the discriminator class, which determines the data-coverage assumption in the case of value-function learning. As another surprising observation, the regularizer can be altered to relax the data-coverage requirement, and completely eliminate it in the ideal case with strong side information. %
\end{abstract}

\section{Introduction} %
Off-policy evaluation (OPE) often refers to two related tasks in reinforcement learning (RL): estimating the expected return of a target policy using a dataset collected from a different \textit{behavior} policy, versus estimating the policy's value function (or other functions of interest, such as density ratios). The former is crucial to hyperparameter tuning and verifying the performance of a policy before real-world deployment in offline RL \citep{voloshin2019empirical, paine2020hyperparameter, zhang2021towards}. The latter, on the other hand, plays an important role in (both online and offline) training, often as the subroutine of actor-critic-style algorithms \citep{lagoudakis2003least, liu2019off}, but is also generally more difficult than the former: if an accurate value function is available, one could easily estimate the return by plugging in the initial distribution. 

Between the two tasks, the theoretical nature of off-policy return estimation is relatively well understood, especially in terms of the function-approximation assumptions needed for sample-complexity guarantees. Among the available algorithms, importance sampling (IS) and its variants \citep{precup2000eligibility, thomas2015high, jiang2016doubly} do not require any function approximation, but incur \textit{exponential-in-horizon} variance; %
Fitted-Q Evaluation \citep{ernst2005tree,le2019batch} can enjoy polynomial sample complexity under appropriate coverage assumptions, but the guarantee relies on the strong Bellman-completeness assumption on the function class; marginalized importance sampling (MIS) methods, which have gained significant attention recently \citep{liu2018breaking,xie2019towards,uehara2019minimax,nachum2019dualdice}, use two function classes to simultaneously approximate the value and the density-ratio (or weight) function and optimize minimax objectives. Notably, it is the only family of methods known to produce accurate return estimates with a polynomial sample complexity, when the function classes only satisfy the relatively weak realizability assumptions (i.e., they contain the true value and weight functions). %

In comparison, little is known about off-policy function estimation, and the guarantees are generally less desirable. Not only do the limitations of IS and FQE on return estimation carry over to this more challenging task, but MIS also loses its major advantage over FQE: despite the somewhat misleading impression left by many prior works, that MIS can handle function estimation the same way as return estimation,\footnote{For example, \citet{liu2019off} assumed a weight estimation oracle and cited \citet{liu2018breaking} as a possible instance.} %
MIS for function estimation often requires unrealistic assumptions  such as prohibitively expressive discriminators. %
For concreteness, a typical guarantee for function estimation from MIS looks like the following (see e.g., Theorem 4 of \citet{liu2018breaking} and Lemmas 1 and 3 of \citet{uehara2019minimax}): 

\begin{proposition}[Function estimation guarantee for MIS, informal] \label{prop:rich_disc}
Suppose the offline data distribution $d^D$ satisfies $d^D(s,a) > 0, \forall s, a$. Given value-function class $\Qcal$ with $q^\pi \in \Qcal$ and weight class $\Wcal = \RR^{\Scal\times\Acal}$, $q^\pi = \argmin_{q\in\Qcal} \max_{w\in \Wcal} L(w, q)$ for some appropriate population loss function $L$.\footnote{Concretely, one can choose $L$ as Eq.\eqref{eq:lagrangian_reg_q} without the regularization term, which recovers MQL in \citet{uehara2019minimax}.}
\end{proposition} 

To enable the identification of the value function $q^\pi$, the result %
requires the discriminator class $\Wcal$ to be the \textit{space of all possible functions over the state-action space} ($\Wcal= \RR^{\Scal\times\Acal}$). In the finite-sample regime, using such a class incurs a sample complexity that depends on the size of the state-action space, which completely defeats the purpose of function approximation. %

In addition, these results only hold asymptotically, where the function of interest can be exactly identified in a point-wise manner. Such an overly strong guarantee is unrealistic in the finite-sample regime, where one can only hope to approximate the function well in an average sense \textit{under some distribution}, i.e., finite-sample performance guarantees should ideally bound $\|\wh{q} - q^\pi \|_{2, \nu}$ for the learned $\wh{q}$, where $\|\cdot\|_{2, \nu}$ is $\nu$-weighted 2-norm. Such fine-grained analyses are non-existent in MIS. Even in the broader literature, such results not only require Bellman-completeness-type assumptions \citep{uehara2021finite}, they also come with some fixed $\nu$ (which is not necessarily $d^D$; see Section~\ref{sec:related}) and the user has no freedom in choosing $\nu$. This creates a gap in the literature, as downstream learning algorithms that use off-policy function estimation as a subroutine often assume the estimation to be accurate under specific distributions \citep{kakade2002approximately, abbasi2019politex} (see details in Appendix~\ref{app:discuss}). 

To summarize, below are two important open problems on off-policy function estimation:
\begin{enumerate}[leftmargin=*]
\item \textit{Is it possible to obtain polynomial\footnote{By ``polynomial'', we mean polynomial in the horizon, the statistical capacities and the boundedness of the function classes, and the parameter that measures the degree of data coverage.} sample complexity for off-policy function estimation, using function classes that only satisfy realizability-type assumptions?}
\item \textit{Can we specify a distribution $\nu$ to the estimation algorithm, such that the learned function will be close to the groundtruth under $\nu$?}
\end{enumerate}
In this work, we answer both open questions in the positive. By imposing proper regularization on the MIS objectives, we provide off-policy function estimation guarantees under only realizability assumptions on the function classes. Compared to commonly used regularization in MIS \citep{nachum2019dualdice, nachum2020reinforcement, yang2020off}, 
our regularizer is much more flexible and can account for an arbitrary user-specified distribution $\nu$, under which the learned function will be close to the groundtruth. We  provide  \textit{exact} characterization of the optimal dual solution that needs to be realized by the discriminator, which determines the data-coverage assumption in value-function learning. As another surprising observation, the regularizer can be altered to relax the data-coverage requirement, and in the ideal case \textit{completely eliminate} it when strong side information is available. Proof-of-concept experiments are also conducted to validate our theoretical predictions.

\section{Related Works} \label{sec:related}
\para{Regularization in MIS} The use of regularization is very common in the MIS literature, especially in DICE algorithms \citep{nachum2019dualdice,nachum2019algaedice,yang2020off}. However, most prior works that consider regularization use tabular derivations and seldom provide finite-sample function-approximation guarantees on even return estimation, let alone function estimation. (An exception is the work of \citet{uehara2021finite}, who analyze related estimators under Bellman-completeness-type assumptions; see the next paragraph.)  
More importantly, prior works provide very limited understanding in how choice of regularization affects learning guarantees, 
and have considered only na\"ive forms of regularization (state-action-\textit{independent} and typically under $d^D$), 
under which different forms of regularization are essentially treated equally under a coarse-grained theory \citep{yang2020off}.
In contrast, we provide much more fine-grained characterization of the effects of regularization, which leads to novel insights about how to design better regularizers, 
and existing DICE estimators are subsumed as special cases of our method when we choose very simple regularizers (see Remark~\ref{rem:dualdice} in Section~\ref{sec:weight}). 

\para{Fitted-Q Evaluation (FQE)} Outside the MIS literature, one can obtain return \textit{and} value-function estimation guarantees via FQE \citep{duan2020minimax, chen2019information, le2019batch, uehara2021finite}. However, it is well understood that FQE and related approaches require   Bellman-completeness-type assumptions, such as the function class being \textit{closed} under the Bellman operator. %
Even putting aside the difference between completeness vs.~realizability, we allow for a user-specified error-measuring distribution, which is not available in FQE or any other existing method. %
The only distribution these methods are aware of is the data distribution $d^D$, and even so, FQE and variants rarely provide guarantees on $\|\wh{q} - q^\pi\|_{2, d^D}$, but often on the Bellman error (e.g., $\|\wh{q} - \Tcal^\pi \wh{q}\|_{2, d^D}$) instead \citep{uehara2021finite}, and obtaining guarantees on a distribution of interest often requires multiple indirect translations and loose relaxations.

\para{LSTDQ} Our analyses focus on general function approximation. %
When restricted to linear classes, function estimation guarantees for $q^\pi$ under $d^D$ can be obtained by LSTDQ methods \citep{lagoudakis2003least, bertsekas2009projected, dann2014policy} when the function class only satisfies realizability of $q^\pi$ \citep{perdomo2022sharp}. However, this requires an additional matrix invertibility condition (see Assumption 3 of \citet{perdomo2022sharp}), and it is still unclear what this condition corresponds to in general function approximation.\footnote{It is hinted by \citet{uehara2019minimax} that the invertibility is related to a loss minimization condition in MIS, but the connection only holds for return estimation.} Moreover, many general methods---including MIS \citep{uehara2019minimax} and other minimax methods \citep{antos2008learning, xie2021bellman}---coincide with LSTDQ in the linear case, so the aforementioned results can be viewed as a specialized analysis leveraging the properties of linear classes.

\para{PRO-RL \citep{zhan2022offline}} Our key proof techniques are adapted from \citet{zhan2022offline}, whose goal is offline policy learning. They learn the importance weight function $w^\pi$ for a near-optimal $\pi$, and provide $\|\wh{w} - w^\pi\|_{2, d^D}$ guarantees as an intermediate result. Despite using similar technical tools, our most interesting and surprising results  are in the value-function estimation setting, which is not considered by \citet{zhan2022offline}. Our novel algorithmic insights, such as incorporating error-measuring distributions and approximate models in the regularizers, are also potentially useful in \citet{zhan2022offline}'s policy learning setting. %
Our analyses also reveal a number of important differences between OPE and offline policy learning, which will be discussed in Appendix~\ref{app:discuss}.

\section{Preliminaries}
We consider %
off-policy evaluation (OPE) in Markov Decision Processes (MDPs). An MDP is specified by its state space $\Scal$, action space $\Acal$, transition dynamics $P: \Scal\times\Acal\to\Delta(\Scal)$ ($\Delta(\cdot)$ is the probability simplex), reward function $R: \Scal\times\Acal \to \Delta([0, 1])$, discount factor $\gamma \in [0, 1)$, and an initial state distribution $\init \in \Delta(\Scal)$. We assume $\Scal$ and $\Acal$ are finite and discrete, but their cardinalities can be arbitrarily large. %
Given a target policy $\pi: \Scal\to\Delta(\Acal)$,  a random trajectory $s_0, a_0, r_0, s_1, a_1, r_1, \ldots$ can be generated as $s_0 \sim \init, a_t\sim \pi(\cdot|s_t), r_t \sim R(\cdot|s_t, a_t), s_{t+1} \sim P(\cdot|s_t, a_t)$, $\forall t\ge 0$; we use $\EE_\pi$ and $\pp_\pi$ to refer to expectation and probability under such a distribution. The expected discounted return (or simply return) of $\pi$ is $J(\pi):=\EE_\pi[\sum_{t}\gamma^{t}r_t]$. 
The Q-value function of $\pi$ is the unique solution of the Bellman equations $q^\pi = \Tcal^\pi q^\pi$, with the Bellman operator $\Tcal^\pi: \RR^{\Scal\times\Acal} \to \RR^{\Scal\times\Acal}$ defined as $\forall q\in \RR^{\Scal\times\Acal}, (\Tcal^\pi q)(s,a) := \EE_{r \sim R(\cdot|s,a)}[r] + \gamma (P^{\pi} q)(s,a)$. Here $P^\pi \in \RR^{|\Scal\times\Acal|\times |\Scal\times\Acal|}$ is the state-action transition operator of $\pi$, defined as 
$(P^\pi q)(s,a) := \E_{s'  \sim P(\cdot|s,a), a'\sim \pi(\cdot|s')}[q(s',a')]$. Functions over $\Scal\times\Acal$ (such as $q$) are also treated as $|\Scal\times\Acal|$-dimensional vectors interchangeably.%

In OPE, we want to estimate $q^\pi$ and other functions of interest based on a historical dataset collected by a possibly different policy. As a standard simplification, we assume that the offline dataset consisting of $n$ i.i.d.~tuples $\{(s_i,a_i,r_i,s_i')\}_{i=1}^n$ sampled as $(s_i,a_i) \sim d^D$, $r \sim R(\cdot | s_i,a_i)$, and $s_i'\sim P(\cdot|s_i,a_i)$. We call $d^D$ the (offline) data distribution. As another function of interest, the (marginalized importance) weight function $w^\pi$ is defined as $w^\pi(s,a) := d^\pi(s,a)/d^D(s,a)$, where $d^\pi(s,a) = (1-\gamma) \sum_{t=0}^\infty \gamma^t \pp_\pi[s_t = s, a_t = a]$ is the discounted state-action occupancy of $\pi$. For technical convenience we assume $d^D(s,a) >0 ~\forall s, a$, so that quantities like $w^\pi$ are always well defined and finite.\footnote{It will be trivial to remove this assumption at the cost of cumbersome derivations. Also, these density ratios can still take prohibitively large values even if they are finite, and we will need to make additional boundedness assumptions to enable finite-sample guarantees anyway, so their finiteness does not trivialize the analyses.\label{footnote:data_coverage}} Similarly to $q^\pi$, $w^\pi$ also satisfies a recursive equation, inherited from the Bellman flow equation for $d^\pi$: $d^\pi = (1-\gamma)\initsa + \gamma \Pt d^\pi$, where $(s,a)\sim \initsa \Leftrightarrow s\sim \init, a \sim \pi(\cdot|s)$ is the initial state-action distribution, and $\Pt := (P^\pi)^\top$ is the transpose of the transition matrix.

\para{Function Approximation} We will use function classes $\Qcal$ and $\Wcal$ to approximate $q^\pi$ and $w^\pi$, respectively. We assume finite $\Qcal$ and $\Wcal$, and extension to infinite classes under appropriate complexity measures (e.g., covering number) is provided in Appendix \ref{appendix:infinite}. 

\para{Additional Notation} $\|\cdot\|_{2, \nu}:= \sqrt{\EE_{\nu}[(\cdot)^2]}$ is the weighted 2-norm of a function under distribution $\nu$. We also use a standard shorthand $f(s, \pi) := \EE_{a \sim \pi(\cdot|s)}[q(s, a)]$.  Elementwise multiplication between two vectors $u$ and $v$ of the same dimension is $u \emult v$, and elementwise division is $u \ediv v$.  

\section{Value-function Estimation} \label{sec:value}
In this section we show how to estimate $\wh{q} \approx q^\pi$ with guarantees on $\|\wh{q} - q^\pi\|_{2, \nu}$ for a user-specified $\nu$, and identify the assumptions under which provable sample-complexity guarantees can be obtained. 
We begin with the familiar Bellman equations, that $q^\pi$ is the unique solution to: 
\begin{align} \label{eq:q_bellman}
\EE_{r \sim R(\cdot|s,a)}[r] + \gamma \EE_{s'\sim P(\cdot|s,a)}[q(s',\pi)] - q(s,a)  = 0, ~ \forall s, a\in\Scal\times\Acal. 
\end{align}
While the above set of equations uniquely determines $q = q^\pi$, this is only true if we can enforce \textit{all} the $|\Scal\times\Acal|$ constraints, which is intractable in large state-space problems. In fact, even estimating (a candidate $q$'s violation of) a single constraint  is infeasible as that requires sampling from the same state multiple times, which is related to the infamous double-sampling problem \citep{baird1995residual}. 

To overcome this challenge, prior MIS works often relax Eq.\eqref{eq:q_bellman} by taking a \textit{weighted combination} of these equations, e.g.,
\begin{align} \label{eq:mql}
\E_{d^D}\lb w(s,a) \lp r(s,a) + \gamma q(s',\pi) - q(s,a) \rp \rb = 0, ~ \forall w\in \Wcal.
\end{align}
Instead of enforcing $|\Scal\times\Acal|$ equations, we only enforce their linear combinations; the linear coefficients are %
$d^D(s,a) \cdot w(s,a)$, and $w$ belongs to a class $\Wcal$ with limited statistical capacity to enable sample-efficient estimation. While each constraint in Eq.\eqref{eq:mql} can now be efficiently checked on data, this comes with a big cost that a solution to Eq.\eqref{eq:mql} is not necessarily $q^\pi$. Prior works handle this dilemma by aiming lower: instead of learning $\wh{q} \approx q^\pi$, they only learn $\wh{q}$ that can approximate the policy's return, i.e. $\E_{s\sim \init}[\wh{q}(s, \pi)] \approx J(\pi) = \E_{s\sim \init}[q^\pi(s, \pi)]$. While \citet{uehara2019minimax} show that the latter is possible when 
$w^\pi \in \Wcal$, they also show explicit counterexamples   %
where $\wh{q}\neq q^\pi$ even with infinite data. 
As a result, how to estimate $\wh{q} \approx q^\pi$ under comparable realizability assumptions (instead of the prohibitive  $\Wcal=\RR^{\Scal\times\Acal}$ as in Proposition~\ref{prop:rich_disc}) is still an open problem. 

\subsection{Estimator}\label{sec:value_estimator}
We now describe our approach to solving this problem. Recall that the goal is to obtain error bounds for $\|\wh{q} - q^\pi\|_{2, \nu}$ for some distribution $\nu \in \Delta(\Scal\times\Acal)$ specified by the user. Note that we do \textit{not} require information about $r$ and $s'$ that are generated after $(s,a) \sim \nu$ and only care about the  $(s,a)$ marginal itself, so the user can pick $\nu$ in an arbitrary manner without knowing the transition and the reward functions of the MDP. 
We assume that $\nu$ is given in a way that we can take its expectation $\EE_{(s,a)\sim \nu}[(\cdot)]$, and extension to the case where $\nu$ is given via samples is straightforward.

To achieve this goal, we first turn Eq.\eqref{eq:q_bellman} into an equivalent \textit{constrained convex program}: given a collection of \textit{strongly convex}  and differentiable functions $\{f_{s,a}:\RR \to \RR\}_{s,a}$---we will refer to the collection as $f$ and discuss its choice later---consider 
\begin{align}
\min_q \; &\E_{(s,a) \sim \nu}[f_{s,a}(q(s,a))] \label{eq:primal_reg_q} \\ 
\text{s.t.}\; &\EE_{r \sim R(\cdot|s,a)}[r] + \gamma \EE_{s'\sim P(\cdot|s,a)}[q(s',\pi)] - q(s,a)  = 0 , ~\forall s, a \in \Scal\times\Acal. \nonumber
\end{align}
The constraints here are the same as Eq.\eqref{eq:q_bellman}. Since Eq.\eqref{eq:q_bellman} uniquely determines $q = q^\pi$, the feasible space of Eq.\eqref{eq:primal_reg_q} is a \textit{singleton}, so we can impose any objective function on top of these constraints (here we use $\E_{(s,a) \sim \nu}[f_{s,a}(q(s,a))]$)  and it will not change the optimal solution (which is always $q^\pi$, the only feasible point). As we will see, however, $f = \{f_{s,a}:\RR \to \RR\}_{s,a}$ will serve as an important regularizer in the function-approximation setting and is crucial to our estimation guarantees.

\remark[$(s,a)$-dependence of $f$] Regularizers in prior works are  $(s,a)$-\textit{independent} \citep{nachum2019dualdice,yang2020off,zhan2022offline}. As we will see in Section~\ref{sec:coverage}, allowing for $(s,a)$-dependence is very important for designing regularizers with improved guarantees and performances.

We now rewrite \eqref{eq:primal_reg_q} in its Lagrangian form, with $d^D \emult w$ serving the role of dual variables:
\begin{equation}\label{eq:lagrangian_reg_q}
    \min_{q} \max_{w} L_f^q(q,w) := \E_\nu[f_{s,a}(q(s,a))] + \E_{d^D}\lb w(s,a) \lp r(s,a) + \gamma q(s',\pi) - q(s,a) \rp \rb.
\end{equation}
Finally, our actual estimator approximates Eq.\eqref{eq:lagrangian_reg_q} via finite-sample approximation of the population loss $L_f^q$, and searches over restricted function classes $\Qcal$ and $\Wcal$ for $q$ and $w$, respectively:
\begin{equation}\label{eq:sample_lagrangian_q}
    \wh{q} = \argmin_{q \in \Q} \max_{w \in \W} \wh{L}_f^q(q,w),
\end{equation}
where 
$\wh{L}_f^q(q,w) = \E_{\nu}[f_{s,a}(q(s, a))] + \frac{1}{n}\sum_{i=1}^n w(s_i, a_i)\lp r_i + \gamma q(s_i', \pi) - q(s_i,a_i) \rp.$ 

\paragraph{Intuition for identification} Before giving the detailed finite-sample analysis, we  provide some high-level intuitions for why we can obtain the desired guarantee on $\|\wh{q} - q^\pi\|_{2, \nu}$. Note that Eq.\eqref{eq:lagrangian_reg_q} is structurally similar to Eq.\eqref{eq:mql}, and we still cannot verify the Bellman equation for $q^\pi$ in a per-state-action manner, so the caveat of Eq.\eqref{eq:mql} seems to remain; why can we identify $q^\pi$ under $\nu$?

The key here is to show that it suffices to check the loss function $L_f^q$ only under a special choice of $w$ (as opposed to all of $\RR^{\Scal\times\Acal}$). Importantly, this special $w$ is \textit{not} $w=w^\pi$;\footnote{In fact, $w^\pi$ should not appear in our analysis at all: $w^\pi$ is defined w.r.t.~the initial distribution of the MDP, $\init$, which has nothing to do with our goal of bounding $\|\wh{q} - q^\pi\|_{2, \nu}$.} rather, it is the saddle point of our regularized objective $L_f^q$: let $(q^\pi, w_f^*)$ be a saddle point of $L_f^q$ (we will give the closed form of $w_f^*$ later). As long as $w_f^* \in \Wcal$---even if $\Wcal$ is extremely ``simple'' and contains nothing but  $w_f^*$---we can identify $q^\pi$. 

To see that, it is instructive to consider the special case of $\Wcal = \{w_f^*\}$ and the limit of infinite data. In this case, our estimator becomes $\argmin_{q\in\Qcal} L_f^q(q, w_f^*)$. By the definition of saddle point:
$$
L_f^q(q^\pi, w_f^*) \le L_f^q(q, w_f^*), ~\forall q.
$$
While this shows that $q^\pi$ is a minimizer of the loss, it does not imply that it is a unique minimizer. However,  identification immediately follows from the convexity brought by regularization: since $f:\RR \to \RR$ is strongly convex, $q \to \EE_{\nu}[f_{s,a}(q(s,a))]$ as a mapping from $\RR^{\Scal\times\Acal}$ to $\RR$ is strongly convex under $\|\cdot\|_{2, \nu}$ (see Lemma~\ref{lem:strongly_convex} in Appendix~\ref{app:proof_value} for a formal  statement and proof), and $L_f^q(q, w_f^*)$ inherits such convexity since the other terms are affine in $q$. It is then obvious that $q^\pi$ is the unique minimizer of $L_f^q(q^\pi, w_f^*)$ up to $\|\cdot\|_{2, \nu}$, that is, any minimizer of $L_f^q$ must agree with $q^\pi$ on $(s,a)$ pairs supported on $\nu$. Our finite-sample analysis below shows that the above reasoning is robust to finite-sample errors and the inclusion of functions other than $w_f^*$ in $\Wcal$. 

\subsection{Finite-sample Guarantees} \label{sec:q-theorem}
In this subsection we state the formal guarantee of our estimator for $q^\pi$ and the assumptions under which the guarantee holds. We start with the condition on the regularization function $f$:

\begin{assumption}[Strong convexity of $f$]\label{assum:regularizer_q}
Assume $f_{s,a}:\RR\to\RR$ is nonnegative, differentiable, and $M^q$-strongly convex for each $s\in\Scal, a\in\Acal$. In addition, assume both $f_{s,a}$ and its derivative $f_{s,a}'$ take finite values for any finite input. 
\end{assumption}

This assumption can be concretely satisfied by a simple choice of $f_{s,a}(x) = \frac{1}{2} x^2$, which is independent of $(s,a)$ and yields $M^q = 1$. Alternative choices of $f$ will be discussed in Section~\ref{sec:coverage}. 
Next are the realizability and boundedness of $\Wcal$ and $\Qcal$:
\begin{assumption}[Realizability]\label{assum:realizable_q}
Suppose $w^{*}_f \in \W$,  $q^\pi \in \Q$. 
\end{assumption}

\begin{assumption}[Boundedness of $\W$ and $\Q$]\label{assum:bounded_q}
    Suppose $\W$ and $\Q$ are bounded, that is, \\
$        C_\Q^q  := \max_{q \in \Q} \|q\|_\infty < \infty$, ~~
        $C_\W^q  := \max_{w \in \W} \|w\|_\infty < \infty$.
\end{assumption}

As a remark, Assumption~\ref{assum:realizable_q} implicitly assumes the existence of $w_f^*$. As we will see in Section~\ref{sec:coverage}, the existence and finiteness of $w_f^*$ is automatically guaranteed given the finiteness of $f_{s,a}'$ (Assumption~\ref{assum:regularizer_q}) and $d^D(s,a) >0 ~ \forall s,a$. More importantly, Assumptions \ref{assum:realizable_q} and \ref{assum:bounded_q} together imply that $\|q^\pi\|_\infty \le C_\Q^q$ and $\|w_f^*\|_\infty \le C_\W^q$, which puts constraints on how small $C_\Q^q$ and $C_\W^q$ can be. For example, it is common to assume that $C_\Q^q = \tfrac{1}{1-\gamma}$, i.e., the maximum possible return when rewards are bounded in $[0, 1]$, and this way $\|q^\pi\|_\infty \le C_\Q^q$ will hold automatically. The magnitude of $\|w_f^*\|_\infty$ and $C_{\Wcal}^q$, however, is more nuanced and interesting, and we defer the discussion to Section~\ref{sec:coverage}.

Now we are ready to state the main guarantee for identifying $q^\pi$. All proofs of this section can be found in Appendix~\ref{app:proof_value}. %
\begin{theorem}\label{thm:q_bound}
    Suppose Assumptions~\ref{assum:regularizer_q},~\ref{assum:realizable_q},~\ref{assum:bounded_q} hold. Then, with probability at least $1-\delta$, 
    \begin{align*}
        ||\wh{q} - q^\pi||_{2,\nu} \leq 2\sqrt{\frac{\epsilon_{stat}^q}{M^q}}, %
    \end{align*}
    where $\epsilon_{stat}^q = \lp C_\W^q + (1+\gamma) C_\W^q C_\Q^q \rp \sqrt{\frac{2\log \frac{2|\W||\Q|}{\delta}}{n}}$. 
\end{theorem}
Theorem~\ref{thm:q_bound} shows the desired bound on $\|\wh{q} - q^\pi\|_{2, \nu}$, which depends on the magnitude of functions in $\Wcal$ and $\Qcal$ as well as their logarithmic cardinalities, which are standard measures of statistical complexity for finite classes. One notable weakness is the $O(n^{-\nicefrac{1}{4}})$ slow rate; this is due to translating the  $\epsilon_{stat}^q=O(n^{-\nicefrac{1}{2}})$ deviation between $L$ and $\wh{L}$ into $\|\wh{q} - q^\pi\|_{2, \nu}$ via a convexity argument, which takes  square root of the error. The possibility of and obstacles to obtaining an $O(n^{-\nicefrac{1}{2}})$ rate will be discussed in Section~\ref{sec:conclusion}.

\subsection{On the Closed Form of $w_f^*$ and the Data Coverage Assumptions}
\label{sec:coverage}

One unusual aspect of our guarantees in Section~\ref{sec:q-theorem} is that we do not make any explicit data coverage assumptions, yet such assumptions are known to be necessary even for return estimation (typically the boundedness of $w^\pi = d^\pi/d^D$). 
Indeed, our data-coverage assumption is implicit in Assumptions~\ref{assum:realizable_q} and \ref{assum:bounded_q}, which require $\|w_f^*\|_\infty \le C_\W^q < \infty$. 
If data fails to provide sufficient coverage, $\|w_f^*\|_\infty$ will be large and our bound in Theorem~\ref{thm:q_bound} will suffer due to a large value of $C_\W^q$.

To make the data coverage assumption explicit, we provide the closed-form expression of $w_f^*$:
\begin{lemma}\label{lem:wfstar_closedform}
The saddle point of \eqref{eq:lagrangian_reg_q} is $(q^\pi, w_f^*) = \argmin_{q} \argmax_{w} L_f^q(q, w)$, where
\begin{equation}\label{eq:wfstar}
w_f^* =   (I - \gamma \Pt)^{-1} \lp\nu \emult f'(q^\pi) \rp \ediv d^D.
\end{equation}
Here $f'(q^\pi)$ is the shorthand for $[f'_{s,a}(q^\pi(s,a))]_{s,a} \in \RR^{\Scal\times\Acal}$.
\end{lemma}

The closed-form expression in Eq.\eqref{eq:wfstar} looks very much like a density ratio: if we replace $\nu \emult f'(q^\pi)$ with $\initsa$, we have $(I - \gamma \Pt)^{-1} \initsa = d^\pi/(1-\gamma)$, and the expression would be the ratio between $d^\pi$ and $d^D$ (up to a horizon factor). Therefore, $w_f^*$ can be viewed as the density ratio of $\pi$ against $d^D$ when $\pi$ starts from the ``fake'' initial distribution $\nu \emult f'(q^\pi)$. However, $\nu \emult f'(q^\pi)$ is in general not a valid distribution, as it is not necessarily normalized or even non-negative, making $\|w_f^*\|_\infty$ difficult to intuit. Below we give relaxations of $\|w_f^*\|_\infty$, which are more interpretable and give novel insights into how to relax the data-coverage assumption via tweaking $f$. 

\begin{proposition} 
\label{prop:wfstar_bound}
$\|w_f^*\|_\infty \le \tfrac{1}{1-\gamma}\cdot  \|d_\nu^\pi / d^D\|_\infty \cdot \|f'(q^\pi)\|_\infty$, where $d_{\nu}^\pi$ is the discounted state-action occupancy of $\pi$ under $\nu$ as the initial state-action distribution. 
\end{proposition}
The proposition states that $\|w_f^*\|_\infty$ can be bounded if data provides sufficient coverage over $d_\nu^\pi$, and if $f'(q^\pi)$ is bounded. The former shows that $d^D$ needs to cover not only $\nu$, but also state-action pairs reachable by $\pi$ starting from $\nu$.
The latter is easily satisfied, and can be bounded again for concrete choices of $f$, e.g. $\|f'(q^\pi)\|_\infty \le \|q^\pi\|_\infty \leq \frac{1}{1-\gamma}$ for $f_{s,a}(x) = \tfrac{1}{2} x^2$. 

\paragraph{Designing $f$ to relax the coverage assumption} Lemma~\ref{lem:wfstar_closedform} shows that the coverage assumption (bounded $\|w_f^*\|_\infty$) depends on $f$ (or rather its derivative $f'$), which opens up the possibility of properly designing $f$ to relax it. In fact, we could completely eliminate the coverage assumption if we could set $f'(q^\pi) = \mathbf{0}$, but that would require unrealistically strong side information. 

As a concrete example, consider $f_{s,a}(x) = \tfrac{1}{2}(x - q^\pi(s,a))^2$, and it is easy to verify that $f'_{s,a}(q^\pi(s,a)) = x - q^\pi(s,a) |_{x = q^\pi(s,a)} = 0$. Compared to $f_{s,a}(x) = \tfrac{1}{2} x^2$, the new $f$ essentially adds a 1st-order term $q^\pi(s,a) \cdot x$ to change $w_f^*$, while leaving the convexity required by Assumption~\ref{assum:regularizer_q} intact, which only depends on the 2nd-order term $\tfrac{1}{2}x^2$. 
Of course, this is not a viable choice of $f$ in practice as it requires knowledge of $q^\pi$, which is precisely our learning target. %

While the reason $f_{s,a}(x) = \tfrac{1}{2}(x - q^\pi(s,a))^2$ can eliminate the coverage requirement is obvious retrospectively ($q^\pi$ already minimizes $\EE_\nu[f(q)]$ even without any data), our analyses apply much more generally and characterize the effects of arbitrary $f$ on the coverage assumption. Inspired by this example, 
we can consider practically feasible choices such as $f_{s,a}(x) = \tfrac{1}{2}(x-\widetilde{q}(s,a))^2$, where $\widetilde{q}$ is an approximation of $q^\pi$ obtained by other means, e.g. a guess based on domain knowledge. If $\widetilde{q} \approx q^\pi$, our estimator enjoys significantly relaxed coverage requirements. But even if $\widetilde{q}$ is a poor approximation of $q^\pi$, it does not affect our estimation guarantees as long as the condition implied by Proposition~\ref{prop:wfstar_bound} is satisfied. (In fact, $f_{s,a}(x) = \tfrac{1}{2}x^2$ is a special case of $\widetilde{q} \equiv 0$.) Such a use of approximate models is similar to how doubly robust estimators \citep{dudik2011doubly, jiang2016doubly,thomas2016data} enjoy reduced variance given an accurate model, and remain unbiased even if the approximate model is arbitrarily poor. We will show in Section~\ref{sec:experiments} that this idea is empirically effective. 

\section{Weight-function Estimation} \label{sec:weight}
Similar to value-function estimation, our methodology can also be applied to estimate the weight function $w^\pi$. Due to the similarity with Section~\ref{sec:value} in the high-level spirit, we will be concise in this section and only explain in detail when there is a conceptual difference from Section~\ref{sec:value}. Some notations (such as the function classes $\Wcal$ and $\Qcal$) will be abused, but we emphasize that this section considers a different learning task than Section~\ref{sec:value}, so they should be viewed as different objects (e.g., the realizability assumptions for $\Wcal$ and $\Qcal$ below will be different from those in Section~\ref{sec:value}). 

As before, we assume that the user provides a distribution\footnote{Recall we assume $d^D(s,a)>0 ~ \forall s, a$ for technical convenience. When this is not the case, $\eta$ should be supported on $d^D$, as the target function $w^\pi$ is only defined on the support of $d^D$.} $\eta\in\Delta(\Scal\times\Acal)$ and our goal is to develop an estimator with guarantees on $\|\wh{w} - w^\pi\|_{2, \eta}$. 
Analogous to Section~\ref{sec:value}, consider %
\begin{align}
    \min_w \; &\E_{(s,a) \sim \eta}[f_{s,a}(w(s,a))] \label{eq:dual_reg_w} \\ 
    \text{s.t.}\; & \textstyle d^D(s,a)w(s,a) = (1-\gamma)\mu_0^\pi(s,a) + \gamma \sum_{s',a'}P^\pi(s,a|s',a')d^D(s',a')w(s,a), ~\forall s,a. \nonumber
\end{align}
Here $f = \{f_{s,a}\}_{s,a}$ will need to satisfy similar %
assumptions as in Section~\ref{sec:value}. 
The constraints are the Bellman flow equations with a change of variable $d(s,a) = d^D(s,a) \cdot w(s,a)$. Their unique solution is $d(s,a) = d^\pi(s,a)$ (and hence $w(s,a) = d^\pi(s,a)/d^D(s,a)$), thus the feasible space is again a singleton, and the objective does not alter the optimal solution. %
We then use dual variables $q$ to rewrite~\eqref{eq:dual_reg_w} in its Lagrangian form: $ \min_{w} \max_{q} L_f^w(q,w) := $ 
\begin{equation}\label{eq:lagrangian_reg_w}
   \E_{\eta}\lb f_{s,a}(w(s,a))\rb + (1-\gamma) \E_{\mu_0}\lb q(s,\pi)\rb + \E_{d^D}\lb  w(s,a)(\gamma  q(s',\pi) - q(s,a)) \rb
\end{equation}
We approximate the saddle-point solutions by optimizing the empirical loss $\wh{L}_f^w$ over restricted function classes $\W, \Q$: %
$\wh{w} = \argmin_{w \in \W} \max_{q \in \Q} \wh{L}_f^w(q,w)$, 
where $\wh{L}_f^w(q,w) := \E_{\eta}\lb f_{s,a}(w(s,a))\rb + (1-\gamma)   \frac{1-\gamma}{n_0}\sum_{j=1}^{n_0} q(s_j,\pi) + \frac{1}{n}\sum_{i=1}^n w(s_i, a_i)\lp \gamma q(s_i', \pi) - q(s_i,a_i) \rp$, and %
$\{s_j\}_{j=1}^{n_0}$ is a separate dataset sampled i.i.d.~from $\init$ to provide information about the initial distribution. %

We provide the closed-form expression for the saddle point of $L_f^w$ below, which resembles the Q-function for a proxy reward function $f'(w^\pi) \emult \eta  \ediv d^D$.
\begin{lemma}\label{lem:qfstar_closedform}
    The closed form solutions of~\eqref{eq:lagrangian_reg_w} are $(w^\pi, q_f^*) = \argmin_w \argmax_{q}L_f^w(q,w)$, where 
    \begin{equation}\label{eq:qfstar}
        q_f^* = (I - \gamma P^\pi)^{-1} ( f'(w^\pi) \emult \eta  \ediv d^D).
    \end{equation}
\end{lemma}

\remark[Data Coverage Assumption] As we will see, the only data coverage assumption we need is the boundedness of $w^\pi = d^\pi/d^D$. Since $w^\pi$ is the function of interest and practical algorithms can only output functions of well-bounded ranges, such an assumption is an essential part of the learning task itself and hardly an additional requirement. Moreover, unlike Section~\ref{sec:value}, changing $f$ here will not affect the data-coverage assumption, though it still alters $q_f^*$, and a properly chosen $f$ (e.g., with $f'(w^\pi) \approx \mathbf{0}$) can still result in a $q_f^*$ with small magnitude and thus make learning easier. 

\remark[Connection to DualDICE\label{rem:dualdice}] We can recover DualDICE~\citep{nachum2019dualdice} by choosing $f_{s,a}(x) = \frac{1}{2} x^2$ and $\nu = d^D$. %
Despite producing the same estimator, the derivations and assumptions under which the two works analyze the estimator are different. Their Theorem 2 only provides return estimation guarantees, and depends on an implicit assumption of highly expressive function classes\footnote{In our notation, they measure the approximation error of $\Wcal$ as $\max_{w'\in\RR^{\Scal\times\Acal}} \min_{w\in\Wcal} \|w - w'\|$, essentially requiring $\Wcal$ (and similarly $\Qcal$) to closely approximate every function over $\Scal\times\Acal$. However, we suspect that they could have measured realizability errors instead without changing much of their proofs.} similar to Proposition~\ref{prop:rich_disc}. 
Moreover, they do not characterize how the choice of $f$ can affect the learning guarantees (their $f$ is $(s,a)$-independent). This is one of the main insights of our paper and leads to the discovery of more practical regularizers, e.g. $f_{s,a}(x) = \tfrac{1}{2} (x - \tilde{w}(s,a))^2$ with model $\wt{w}$.

Below we present the assumptions, then learning guarantee for $\wh{w}$. 

\begin{assumption}[Strongly Convex  Objective]\label{assum:regularizer_w}
Suppose for all $s,a$, $f_{s,a}$ is differentiable, non-negative, and $M^w$-strongly convex. Further, suppose $f_{s,a}$ and its derivative take finite values on any finite inputs, and let $C_f^w := \max_{w \in \W}||f(w)||_\infty$. 
\end{assumption}

\begin{assumption}[Realizability]\label{assum:realizable_w}
Suppose $w^{\pi} \in \W$, $q_f^* \in \Q$. 
\end{assumption}

\begin{assumption}[Bounded $\W$ and $\Q$]\label{assum:bounded_w}
    Let $C_\W^w  := \max_{w \in \W} ||w||_\infty$ and $C_\Q^w  := \max_{q \in \Q} ||q||_\infty$. Suppose $\W$ and $\Q$ are bounded function classes, that is, $C_\W^w < \infty$ and $C_\Q^w < \infty$. 
\end{assumption}

\begin{theorem}\label{thm:w_bound}
    Suppose Assumptions~\ref{assum:regularizer_w},~\ref{assum:realizable_w},~\ref{assum:bounded_w}, hold. Then w.p. $\geq 1-\delta$, 
$        \|\wh{w} - w^\pi\|_{2,\eta} \leq 2\sqrt{\frac{\epsilon_{stat}^w}{M^w}}, %
$
    where $\epsilon_{stat}^w = \lp C^w_{f} + (1+\gamma) C_\W^w C_\Q^w \rp \sqrt\frac{2\log \frac{4|\Q||\W|}{\delta}}{n} + (1-\gamma)C_\Q^w \sqrt\frac{2\log \frac{4|\Q|}{\delta}}{n_0}$. 
\end{theorem}

\section{Experiments}\label{sec:experiments}
We now provide experimental results to verify our theoretical predictions and insights. As \citet{yang2020off} have performed extensive experiments on return estimation with simple regularization ($f_{s,a}(x) = \tfrac{1}{2}x^2$), we focus on the task of $q^\pi$ estimation, and the following two questions unique to our work: \\[0.25em]
\textbf{Q1.}~ When the goal is to minimize $\|\wh{q}-q^\pi\|_{2, \nu}$, how much benefit does regularizing with $\nu$ bring in practice, compared to regularizing with other distributions (or no regularization at all)?\\[0.25em]
\textbf{Q2.}~ Can incorporating (even relatively poor) models in regularization (e.g., $f_{s,a}(x) = \tfrac{1}{2}(x - \widetilde{q}(s,a))^2$ from Section~\ref{sec:coverage}) improve estimation?

\para{Setup} We study these questions in a large tabular Gridwalk environment  \cite{nachum2019dualdice, yang2020off}, with a deterministic target policy $\pi$ that is optimal, and a behavior policy that provides limited coverage over the target policy; see Appendix~\ref{appendix:experiments} for further details. To mimic the identification challenges associated with restricted function classes, we use a linear function class $\Q = \{\Phi^\top \alpha : \alpha \in \R^d\}$ and discriminator class $\W = \{\wt\Phi^\top \beta : \beta \in \R^k\}$, where $k < d \ll |\Scal\times \Acal|$. The features $\Phi \in \RR^{|\Scal\times\Acal|\times d}$, $\wt\Phi \in \RR^{|\Scal\times\Acal|\times k}$ are chosen to satisfy the realizability assumptions of all estimators. Under linear classes, our estimator (Eq.\eqref{eq:sample_lagrangian_q}) becomes a convex optimization problem with $d$ variables and $k$ linear constraints, and can be solved by standard packages. %
This allows us to avoid difficult minimax optimization---which is still an open problem in the MIS literature---and focus on the statistical behaviors of our estimators, which is what our theoretical predictions are about. %

\remark When no regularization is used, our linear estimator coincides with MQL~\cite{uehara2019minimax}. If we further had $\wt\Phi = \Phi$, the estimator would coincide with LSTDQ. While Section~\ref{sec:related} mentioned that LSTDQ enjoys function-estimation guarantees \citep{perdomo2022sharp} (and folklore suggests they extend to $\wt\Phi \ne \Phi$), the guarantee only holds in the regime of $k\ge d$, i.e., the $k$ linear constraints are \textit{over-determined}. In our case, however, we have \textit{under-determined} constraints ($k < d$), creating a more challenging learning task (which our theory can handle) where LSTDQ's guarantees do not apply.

\begin{figure}[t]
    \includegraphics[width=\textwidth]{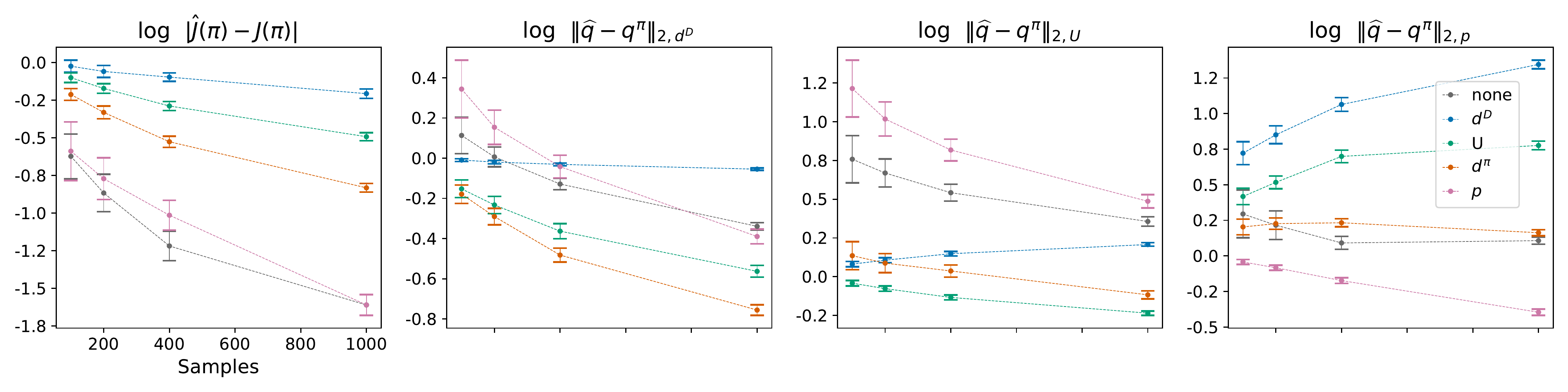} \vspace{-1.5em}
    \caption{Error of off-policy return and function estimation as a function of sample size. Legend shows regularizing distribution $\nu$ and header shows error-measuring distribution $\nu'$ (see text). Error bars show 95\% confidence intervals calculated from 1000 runs.
    \label{fig:sample_error}}
\end{figure}

\begin{figure}[t]
    \centering
    \includegraphics[width=\textwidth]{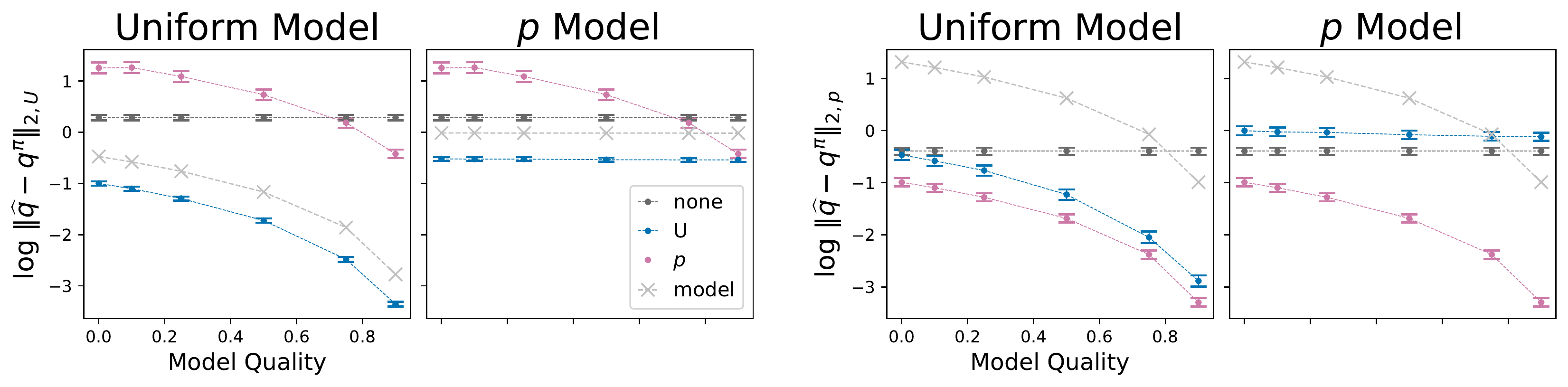}
    \vspace{-1.5em}
    \caption{Estimation error when the regularizer incorporates a model $\wt q$, where x-axes represent the parameter $m$ that controls the quality of $\wt q$. 
    The "model" line shows performance of $\wt{q}$. 
    Sample size is 500 and the results are from 500 runs.     \label{fig:model_error}}
\end{figure}

\para{Choice of Distributions} We consider a set of diverse distributions $\mathcal{V} = \{d^D, \mu_0^\pi, d^\pi, U, p\}$, where $U$ is uniform over $\Scal\times\Acal$ and $p \propto (d^\pi \emult \mathbb{I}[w^\pi > 50])$. The distribution $p$ isolates the least-covered states reached by $\pi$, which makes learning an accurate Q-function on $\nu$ a harder task. 

\para{Results for Q1} We use a default regularizer $f = \tfrac{1}{2} x^2$ with different regularizing distributions $\nu \in \mathcal{V}$, and measure $\|\wh{q} - q^\pi\|_{2, \nu'}$ for different $\nu'$. The results are shown in Figure~\ref{fig:sample_error}, which exhibit the expected trend: for example, regularizing with $\nu=p$ performs poorly when the error is measured under $\nu' = d^D$ and $U$ due to the large mismatch between $\nu$ and $\nu'$. However, when $\nu' = p$ (rightmost panel), regularizing with $\nu=p$ significantly outperforms others. Similar behaviors can also be observed on $U$, though they are certainly not absolute (e.g., $\nu = d^D$ does not do very well on $\nu' = d^D$), which suggests potential directions for more refined and accurate theory. Moreover, %
using no regularization (``none'') generally does not perform well for any $\nu'$, but still manages to achieve a high accuracy for return estimation $J(\pi)$, which is consistent with prior theory \citep[e.g.,][]{uehara2019minimax} that return estimation does not require regularization.

\para{Results for Q2} We now use $f_{s,a}(x) = \tfrac{1}{2}(x - \wt q(s,a))^2$ with different $\widetilde{q}$ to verify how the quality of $\widetilde{q}$ affect estimation accuracy. We first consider a ``uniform model'' $\wt{q} = mq^\pi + (1-m)\overline{q}$, where $\overline{q}$ is a constant and $m \in [0, 1]$ controls the quality $\wt{q}$. As shown from Panels 1 \& 3 in Figure~\ref{fig:model_error}, our estimator's accuracy generally improves with a better $\wt q$ (i.e., as $m$ increases). Moreover, equipping $\wt q$ with an appropriate regularizing distribution $\nu$ (e.g., $\nu = U$ for both panels) can significantly outperform no regularization, even with a very poor $\wt q$ (e.g., $m=0.1$). It also outperforms the model prediction itself (i.e., $\wh q = \wt q$), showing that the improvement is not due to our estimator simply taking predictions from $\wh q$, but using the regularization to better assist the identification of $q^\pi$ from data. 

The previous model's quality is uniform across $\Scal\times\Acal$. We then consider a scenario where $\wt{q}$ is zeroed out outside $p$'s support, making it only a good approximation of $q^\pi$ on $p$. In this case, we see that regularization cannot benefit much from the model when the error is measured on $\nu' = U$ (Panel 2), but when $\nu'=\nu=p$ (Panel 4), regularization can still bring benefits, as expected from our theory.

\section{Discussion and Conclusion} \label{sec:conclusion}
In this paper we showed that proper regularization can yield function-estimation guarantees for MIS methods under only realizable function approximation. Compared to prior works, our regularizer is more flexible and can accommodate a user-specified error-measuring distribution. Further theoretical investigation provides fine-grained characterization of how the choice of regularization affects learning guarantees, which leads to the discovery of regularizers that incorporate approximate models (such as $\wt q$). While the superiority of such regularizers is perhaps obvious retrospectively, it is not allowed in the prior works' derivation that assumes $(s,a)$-independent regularization, and our theoretical results provide a deep understanding for even more general regularization schemes. 
In Appendix~\ref{app:discuss}, we provide further discussions on two topics: (1) the barriers to obtaining a faster $O(n^{-\nicefrac{1}{2}})$ rate, and (2) comparison to \citet{zhan2022offline} reveals interesting differences between off-policy function estimation and policy learning, and insights in this paper may also be useful for the policy learning task.

\begin{ack}
The authors thank Jinglin Chen, Wenhao Zhan, and Jason Lee for valuable discussions during the early phase of the project, and Anonymous Reviewer e62P for insightful feedback that helped improve this paper during the review process. NJ acknowledges funding support from ARL Cooperative Agreement W911NF-17-2-0196, NSF IIS-2112471, NSF CAREER IIS-2141781, and Adobe Data Science Research Award.
\end{ack}

\printbibliography 
\section*{Checklist}

\begin{enumerate}

\item For all authors...
\begin{enumerate}
  \item Do the main claims made in the abstract and introduction accurately reflect the paper's contributions and scope?
    \answerYes{}{}
  \item Did you describe the limitations of your work?
    \answerYes{We discuss the slow rate of our guarantee in Appendix~\ref{app:discuss}.}
  \item Did you discuss any potential negative societal impacts of your work?
    \answerNA{}
  \item Have you read the ethics review guidelines and ensured that your paper conforms to them?
    \answerYes{}
\end{enumerate}

\item If you are including theoretical results...
\begin{enumerate}
  \item Did you state the full set of assumptions of all theoretical results?
    \answerYes{}
        \item Did you include complete proofs of all theoretical results?
    \answerYes{}
\end{enumerate}

\item If you ran experiments...
\begin{enumerate}
  \item Did you include the code, data, and instructions needed to reproduce the main experimental results (either in the supplemental material or as a URL)?
    \answerYes{}
  \item Did you specify all the training details (e.g., data splits, hyperparameters, how they were chosen)?
    \answerYes{}
        \item Did you report error bars (e.g., with respect to the random seed after running experiments multiple times)?
    \answerYes{}
        \item Did you include the total amount of compute and the type of resources used (e.g., type of GPUs, internal cluster, or cloud provider)?
    \answerNo{Our experiments are on relatively simple synthetic environments, and we did not particularly track the use of computation.}
\end{enumerate}

\item If you are using existing assets (e.g., code, data, models) or curating/releasing new assets...
\begin{enumerate}
  \item If your work uses existing assets, did you cite the creators?
    \answerNA{}
  \item Did you mention the license of the assets?
    \answerNA{}
  \item Did you include any new assets either in the supplemental material or as a URL?
    \answerNA{}
  \item Did you discuss whether and how consent was obtained from people whose data you're using/curating?
    \answerNA{}
  \item Did you discuss whether the data you are using/curating contains personally identifiable information or offensive content?
    \answerNA{}
\end{enumerate}

\item If you used crowdsourcing or conducted research with human subjects...
\begin{enumerate}
  \item Did you include the full text of instructions given to participants and screenshots, if applicable?
    \answerNA{}
  \item Did you describe any potential participant risks, with links to Institutional Review Board (IRB) approvals, if applicable?
    \answerNA{}
  \item Did you include the estimated hourly wage paid to participants and the total amount spent on participant compensation?
    \answerNA{}
\end{enumerate}

\end{enumerate}

\newpage

\appendix

\section{Further Discussions} \label{app:discuss}

\paragraph{Motivating Examples} 
Below, we give a thorough discussion, through four examples of how off-policy function estimation used in downstream learning algorithms. We highlight the discrepancies between what these algorithms assume about the function estimates and what existing work is able to achieve, and demonstrate how our work closes these gaps. 

\underline{Batch Learning:} \cite{liu2019off} design an off-policy policy gradient algorithm that requires estimating the density-ratio $w^\pi$ to correct the offline data distribution to the on-policy distribution. In their convergence analysis, they assume access to a blackbox $w^\pi$ estimator that is accurate under $d^D$, and refer to \cite{liu2019off} as a possible method. However, as per Proposition~\ref{prop:rich_disc}, \cite{liu2018breaking} and existing works do not provide desirable guarantees for such a task.

\underline{Online Reinforcement Learning:} The seminal paper of \cite{kakade2002approximately} designs the CPI algorithm for on-policy policy improvement, which inspired popular empirical algorithms such as TRPO and PPO. CPI requires an oracle for estimating the advantage function ($\approx$ value function up to offset) accurately under the on-policy distribution, i.e., distribution induced by the current policy (see their Sec 7.1). While this is easy to do by simple squared-loss regression onto on-policy trajectories, it can be sample-inefficient as it fails to leverage off-policy data collected by previous policies. On the other hand running something like TD on all data considers a distribution different from the on-policy one. Our method offers a direct solution: use all data in the Bellman error part of the objective, and only use on-policy trajectories in the regularizer. 

\underline{Online Reinforcement Learning:} \cite{abbasi2019politex} designs a no-regret policy optimization algorithm assuming access to value-function estimation oracles. In their Theorem 5.1, they assume that the oracle outputs an estimate of $q^\pi$ that is accurate under $\nu = d^{\pi^*}$. While $d^{\pi^*}$ is obviously not accessible to us and our method does not apply as-is, %
one might use our theoretical insights to design heuristics, such as up-weighting high-reward states in the offline distribution, as a way to mimic $d^{\pi^*}$.

\underline{Model selection in Offline Return Estimation:} Model selection in offline return estimation: Hyperparameter tuning is a huge practical hurdle in offline return estimation \citep{paine2020hyperparameter}, i.e., all OPE estimators for return estimation (except for importance sampling which has exponential variance) require some form of function approximation, and it is hard to choose the right function class with offline data alone. To address this issue, \cite{zhang2021towards} proposes a model selection process over candidate function estimates of $q^\pi$, which must be provided by base algorithms that perform function estimation.

\paragraph{Function Estimation \& Downstream Tasks} Online algorithms using off-policy function estimation as a subroutine, such as  \citet{kakade2002approximately,abbasi2019politex}, may require the estimates to be accurate on unknown distributions such as $d^\pi$ or $d^{\pi^*}$ (where $\pi^*$ is the optimal policy), which may not be immediately accessible to the user. 
The user may be able to use domain knowledge to ``guess" a distribution $\nu$ close to or covering the unknown distribution of interest. 
Then our guarantees for function estimation over $\nu$ could similarly, with a change of distribution, be converted to guarantees on the true distribution of interest.   
To this end, an important avenue of future work involves a thorough investigation of how our off-policy function estimation method interacts with such downstream learning algorithms, their assumptions, and their guarantees, as well as how our method can tailored to improve downstream tasks. 

\paragraph{Faster rate} One weakness of our result is the $O(n^{-\nicefrac{1}{4}})$ slow rate of estimation. While $O(n^{-\nicefrac{1}{2}})$ generalization error bounds for related stochastic saddle point exist \citep{zhang2021generalization}, they only apply to strongly-convex-strongly-concave problems, whereas our problem is strongly-convex-non-strongly-concave ($L_f^q$ is affine in $w$ and $L_f^w$ is affine in $q$), making the result not directly applicable. One immediate idea is to introduce dual regularization to make our objectives also strongly concave in the discriminator. However, while primal regularization does not change the feasible space and guarantees that the learned function will be $q^\pi$ (or $w^\pi$, respectively), dual regularization \textit{does} change the optimal solution, introducing a bias. This leads to a trade-off between the improvement in error bounds due to strong concavity and the additional bias, and our preliminary investigation shows that an optimal trade-off between the two sources of errors still leads to an $O(n^{-\nicefrac{1}{4}})$ rate. Therefore, improving the rate (if it is possible at all) will  require novel technical tools for the generalization analyses of strongly-convex-non-strongly-concave stochastic saddle point problems, which will be an interesting future direction. 

On a related note, while the rate for estimating $q^\pi$ and $w^\pi$ is only $O(n^{-\nicefrac{1}{4}})$, we can combine them in a doubly robust form to get $O(n^{-\nicefrac{1}{2}})$ rate for return estimation by careful choices of the regularizing distributions $\nu$ and $\eta$; see Appendix~\ref{sec:ope} for details.

\paragraph{Comparison to off-policy learning} As mentioned earlier, our results are enabled by technical tools adapted from \citet{zhan2022offline}, whose work focuses on off-policy policy learning and learns $w^\pi$ for a near-optimal $\pi$ that is accurate under $d^D$ as an intermediate step. While most of our surprising observations are in the value-function learning scenario (Section~\ref{sec:value}), comparing our guarantee for learning $w^\pi$ (Section~\ref{sec:weight}) to that of \citet{zhan2022offline} still yields interesting observations about the difference between off-policy evaluation and learning. Most notably, we do not need to control the strength of regularization in Eq.\eqref{eq:primal_reg_q}, since the feasible space is a singleton and there is no objective before we introduce $\EE_{\nu}[f(q)]$. In contrast, the feasible space is not a singleton in \citet{zhan2022offline} (it is the space of all possible occupancies) and there is already a return optimization objective, so \citet{zhan2022offline} need to carefully control the strength of their regularization. %
As a consequence, \citet{zhan2022offline} obtain $O(n^{-\nicefrac{1}{6}})$ rate, showing how off-policy learning is potentially more difficult than off-policy function estimation. Another interesting difference is related to our exact characterization of $w_f^*$ and $q_f^*$: \citet{zhan2022offline} do not have a closed-form expression for their optimal dual solution. Such a lack of direct characterization leads to requiring additional assumptions to guarantee the boundedness of such variables (see their Assumptions 11 and 12), which is not a problem in our setting. Finally, our analyses lead to novel algorithmic ideas such as using state-action-dependent regularizers and incorporating approximate models in the regularizers, which are potentially also useful for policy learning.

\section{Proofs for Section~\ref{sec:value}}
\label{app:proof_value}
\subsection{Proof of Theorem~\ref{thm:q_bound}}\label{proof:q_bound}
From Assumption~\ref{assum:regularizer_q} and Lemma~\ref{lem:strongly_convex}, we know that the regularization function $\E_\nu[f_{s,a}(q(s,a))]$ is an $M$-strongly convex function in $q$ on the $\|\cdot\|_{2,\nu}$ norm. 
Now consider $L_f^q(q,w_f^*)$, the Lagrangian function~\eqref{eq:lagrangian_reg_q} at the optimal discriminator  $w_f^*$. 
Since $L_f^q(q,w_f^*)$ is composed of the regularization function plus terms that are linear in $q$, $L_f^q(q,w_f^*)$ is also an $M$-strongly convex function in $q$. 

As $(q^\pi, w_f^*)$ is the saddle point solution of $L_f^q$, we know $q^\pi = \argmin_{q}L_f^q(q,w_f^*)$. Then from the strong convexity of $L_f^q$, 
\begin{align*}
    ||\wh{q} - q^\pi||_{2,\nu} &\leq \sqrt{\frac{2\lp L_f^q(\wh{q},w_f^*) - L_f^q(q^\pi,w_f^*)\rp}{  M^q}} &&  \\
    &\leq \sqrt{\frac{4  \epsilon_{stat}^q}{  M^q}},  && \text{(Lemma ~\ref{lem:q_estimation_error})} 
\end{align*}
where $\epsilon_{stat}^q$ is given in Lemma~\ref{lem:q_statistical_error}.

We provide the helper lemmas and their proofs below: 
\begin{lemma}\label{lem:strongly_convex}
Suppose $f_{s,a} : \R \rightarrow \R$ is $M$-strongly convex. Then $\E_{\nu}[f_{s,a}(q(s,a))] : \R^{|SA|} \rightarrow \R$ is $M$-strongly convex on $\| \cdot \|_\nu$. 
\end{lemma}

\begin{proof}
From the strong convexity of $f_{s,a}$, for any $x, y \in \R$, 
$$f_{s,a}(x) - f_{s,a}(y) \leq f'_{s,a}(x)(x-y) - \frac{M}{2}(x-y)^2$$
Then for $q, q' \in \R^{|SA|}$, 
\begin{align*}
    \E_{\nu}&[f_{s,a}(q(s,a))] - \E_{\nu}[f_{s,a}(q'(s,a))] \\ &\leq \E_{\nu}[f'_{s,a}(q(s,a))(q(s,a) - q'(s,a))] - \E_{\nu}[\frac{M}{2}(q(s,a) - q'(s,a))^2] \\ 
    &\leq \E_{\nu}[f'_{s,a}(q(s,a))(q(s,a) - q'(s,a))] - \lp \min_{s,a} \frac{M}{2}\rp \E_\nu[(q(s,a) - q'(s,a))^2] \\ 
    &= \langle \nabla_q \E_{\nu}[f_{s,a}(q(s,a))], q - q'\rangle - \frac{M}{2} \E_\nu[(q(s,a) - q'(s,a))^2]
\end{align*}
since $\nabla_q \E_{\nu}[f_{s,a}(q(s,a))] = \nu \emult f'_{s,a}(q)$, which gives our result. 
\end{proof}

\begin{lemma}\label{lem:q_statistical_error}
    Suppose Assumption~\ref{assum:bounded_q} holds. 
    Then for all $(q,w) \in \Q \times \W$, w.p. $\geq 1-\delta$, 
    \begin{align*}
        |\wh{L}^q_f(q,w) - L^q_f(q,w)| \leq   \epsilon_{stat}^q, 
    \end{align*}
    where $\epsilon_{stat}^q = \lp C_\W^q + (1+\gamma) C_\W^q C_\Q^q \rp \sqrt{\frac{2\log \frac{2|\W||\Q|}{\delta}}{n}}$. 
\end{lemma}

\begin{proof}
From the linearity of the expectation, it is clear that $L_f^q(q,w) = \E[\wh{L}_f^q]$. 
Let $l_i = w(s_i, a_i)\lp r(s_i, a_i) + \gamma q(s_i', \pi) - q(s_i, a_i)\rp$. From Assumption~\ref{assum:bounded_q}, 
\begin{align*}
    |l_i| &\leq \|w\|_\infty + (1+\gamma)\|w\|_\infty \|q\|_\infty \\
    &\leq C_\W^q + (1+\gamma) C_\W^q C_\Q^q
\end{align*}

Then using Hoeffding's inequality with union bound, for all $q,w \in \Q \times \W$, w.p. $\geq 1-\delta$, 
\begin{align*}
    \lv \frac{1}{n} \sum_{i=1}^n l_i - \E_{d^D}[l_i] \rv \leq \lp C_\W^q + (1+\gamma) C_\W^q C_\Q^q \rp \sqrt{\frac{2\log \frac{2|\W||\Q|}{\delta}}{n}} = \epsilon_{stat}^q
\end{align*}
\end{proof}

\begin{lemma}\label{lem:q_estimation_error}
    Under Assumptions~\ref{assum:regularizer_q}, ~\ref{assum:realizable_q},~\ref{assum:bounded_q}, w.p. $\geq 1-\delta$, 
    \begin{align*}
        L_f^q(\wh{q},w_f^*) - L_f^q(q^\pi,w_f^*) \leq 2  \epsilon_{stat}^q. 
    \end{align*}
    where $\epsilon_{stat}^q$ is given in Lemma~\ref{lem:q_statistical_error}. 
\end{lemma}

\begin{proof}
Let $\widehat{w}(q) := \argmax_{w \in \Wcal} \widehat{L}_f^q (q,w)$. We decompose the error as follows:
\begin{align*}
    L_f^q(q^\pi,w_f^{*}) - L_f^q(\wh{q},w_f^{*}) &= L_f^q(q^\pi,w_f^{*}) - L_f^q(q^\pi,\wh{w}(q^\pi)) &&\text{(1)  } \geq 0 \\
    &+ L_f^q(q^\pi,\wh{w}(q^\pi)) - \wh{L}_f^q(q^\pi,\wh{w}(q^\pi))  &&\text{(2)  }\geq -\epsilon_{stat}^q \\ 
    &+ \wh{L}^q_f(q^\pi,\wh{w}(q^\pi)) -  \wh{L}^q_f(\wh{q},\wh{w}(\wh{q}))  &&\text{(3)  }\geq 0 \\ 
    &+ \wh{L}^q_f(\wh{q},\wh{w}(\wh{q})) - \wh{L}^q_f(\wh{q},w_f^*)  &&\text{(4)  }\geq 0 \\
    &+ \wh{L}^q_f(\wh{q},w_f^*)  - L^q_f(\wh{q},w_f^*)  &&\text{(5)  }\geq -\epsilon_{stat}^q 
\end{align*}
Combining the terms gives the result, and we provide a brief justification for each inequality below. Terms (2) and (5) follow from Lemma~\ref{lem:q_statistical_error}.

Term (1) $\geq 0$ since $(q^\pi,w_f^{*})$ is the saddlepoint solution. 

Term (3) $\geq 0$, since $\wh{q} = \argmin_{q \in \Q} \wh{L}_f^q(q,\wh{w}(q))$, and $q^\pi \in \Q$. 

Term (4) $\geq 0$ because $w_f^* \in \W$. 
\end{proof}

\subsection{Proof of Lemma~\ref{lem:wfstar_closedform}}
Since strong duality holds, 
the saddle point $(q^\pi, w_f^*)$ satisfies the KKT conditions. 
Then from stationarity, for all $(s,a)$, 
\begin{align*}
    0 &=   \nu(s,a) f'_{s,a}(q^\pi(s,a)) + \gamma \sum_{s',a'}P^\pi(s,a|s',a')d^D(s',a')w_f^*(s',a') - d^D(s,a)w_f^*(s,a).
\end{align*}
Writing this in matrix form, letting $f'(q^\pi)$ be shorthand for $[f'_{s,a}(q^\pi(s,a))]_{s,a} \in \RR^{\Scal\times\Acal}$, $w_f^*$ must satisfy the equality:
\begin{align*}
    (I - \gamma \Pt) (d^D \emult w_f^*) &=   \nu \emult f'(q^\pi) \quad\Longrightarrow\quad d^D \emult w_f^* =   (I - \gamma \Pt)^{-1} \lp\nu \emult f'(q^\pi) \rp. %
\end{align*}

\subsection{Proof of Proposition~\ref{prop:wfstar_bound}} 
Rearranging the closed form of $w_f^*$ from Lemma~\ref{lem:wfstar_closedform} and taking the absolute value of both sides, 
\begin{align*}
    d^D \emult |w_f^*| &= |(I - \gamma \widetilde{P}^{\pi})^{-1} \lp\nu \emult f'(q^\pi)\rp| \\ 
    &\leq \|f'(q^\pi)\|_\infty |(I - \gamma \widetilde{P}^{\pi})^{-1} \nu| \\
    &= \frac{1}{1-\gamma}\|f'(q^\pi)\|_\infty \cdot d^\pi_\nu
\end{align*}

Then dividing both sides by $d^D$ element-wise, this implies
\begin{align*}
    |w_f^*| &\leq \frac{1}{1-\gamma}\|f'(q^\pi)\|_\infty \cdot (d^\pi_\nu \ediv d^D) \\ 
    &\leq  \frac{1}{1-\gamma}\|f'(q^\pi)\|_\infty \cdot \|d^\pi_\nu \ediv d^D\|_\infty
\end{align*}

As the above inequality holds for all $(s,a)$, 
\begin{align*}
    ||w_f^*||_\infty \leq \frac{1}{1-\gamma}
\|f'(q^\pi)\|_\infty \cdot \|d^\pi_\nu / d^D\|_\infty. 
\end{align*}

\section{Proofs for Section~\ref{sec:weight}}
\label{app:proof_weight}
\subsection{Proof of Lemma~\ref{lem:qfstar_closedform}} 
From the KKT stationarity conditions: 
\begin{align*}
    0 &= d^D(s,a)\lp \gamma \E_{s' \sim P(\cdot|s,a)}\lb q_f^*(s',\pi)\rb - q_f^*(s,a)\rp - \nu(s,a) f'_{s,a}(w^\pi(s,a))
\end{align*}
or in matrix form, letting $f'(w^\pi)$ be shorthand for $[f'_{s,a}(w^\pi(s,a))]_{s,a} \in \RR^{\Scal\times\Acal}$, 
\begin{align*}
    \eta \emult f'(w^\pi) = d^D \emult (I - \gamma P^\pi) q_f^*
\end{align*}
Then $q_f^*$ must satisfy 
$$ (I - \gamma P^\pi) q_f^* = f'(w^\pi) \emult \eta \ediv d^D \quad\Longrightarrow\quad q_f^*  = (I - \gamma P^\pi)^{-1}(f'(w^\pi) \emult \eta \ediv d^D)$$ 

\subsection{Proof of Theorem~\ref{thm:w_bound}}
The proof is of a similar nature as the proof of Theorem~\ref{thm:q_bound} (Appendix~\ref{proof:q_bound}). 
From Assumption~\ref{assum:regularizer_w} and Lemma~\ref{lem:strongly_convex}, we know that that $L_f^w(w,q_f^*)$ is an $  M$-strongly convex function in $w$ on the $||\cdot||_{2,\eta}$ norm. Since $(w^\pi, q_f^*)$ is the saddle point solution of $L_f^w$, from strong convexity we know that the error of $\wh{w}$ is bounded as 
\begin{align*}
    ||\wh{w} - w^\pi||_{2,d^D} &\leq \sqrt{\frac{2\lp L_f^w(w^\pi, q_f^*) - L_f^w(\wh{w}, q_f^*)\rp}{M^w}} \\
    &\leq \sqrt{\frac{4\epsilon_{stat}^w}{M^w}} &&\text{(Lemma ~\ref{lem:w_estimation_error})}, 
\end{align*}
where $\epsilon_{stat}^w$ is given in Lemma~\ref{lem:w_statistical_error}. 

\begin{remark}
    In Theorem~\ref{thm:w_bound} of the main text, there is an additional $O(\nicefrac{C_f^w}{\sqrt{n}})$ term in the statistical error $\epsilon_{stat}^w$, which would arise if the regularization function $\E_\eta[f_{s,a}(w(s,a))]$ were to be estimated from samples. 
    However, we state early on in the paper that we assume the regularizer can be calculated exactly, as sampling is a trivial extension. Correspondingly, the correct expression for the statistical error is:
    $$ \epsilon_{stat}^w = (1+\gamma) C_\W^w C_\Q^w  \sqrt{\nicefrac{2\log \frac{4|\Q||\W|}{\delta}}{n}} + (1-\gamma)C_\Q^w \sqrt{\nicefrac{2\log \frac{4|\Q|}{\delta}}{n_0}},$$
    and, to remain consistent with the rest of the paper, we provide the proof and lemma for this $\epsilon_{stat}^w$ below. 
\end{remark}

\begin{lemma}\label{lem:w_statistical_error}
    Suppose Assumption~\ref{assum:bounded_w} holds. 
    Then for all $(w,q) \in \W \times \Q$, w.p. $\geq 1-\delta$, 
    \begin{align*}
        |\wh{L}^w_f(w,q) - L^w_f(w,q)| \leq \epsilon_{stat}^w, 
    \end{align*}
    where $\epsilon_{stat}^w =  (1+\gamma) C_\W^w C_\Q^w  \sqrt\frac{2\log \frac{4|\Q||\W|}{\delta}}{n} + (1-\gamma)C_\Q^w \sqrt\frac{2\log \frac{4|\Q|}{\delta}}{n_0}$. 
\end{lemma}

\begin{proof}
Let $l_i = w(s_i, a_i)(\gamma q(s_i', \pi) - q(s_i, a_i))$. Using Assumption~\ref{assum:bounded_w}, 
\begin{align*}
    |l_i| &\leq   (1+\gamma)||w||_\infty ||q||_\infty \\
    &\leq  (1+\gamma)C_\W^w C_\Q^w
\end{align*}
Then using Hoeffding's inequality with union bound, w.p. $\geq 1-\delta/2$ we have that for all $w, q \in \W \times \Q$,
\begin{align*}
    \lv \frac{1}{n} \sum_{i=1}^n l_i - \E_{d^D}[l_i] \rv \leq   (1+\gamma)C_\W^w C_\Q^w \sqrt{\frac{2\log \frac{4|\W||\Q|}{\delta}}{n}}
\end{align*}

Similarly, for all $q \in \Q$, w.p. $\geq 1-\delta/2$, 
\begin{align*}
    \lv \frac{1}{n_0} \sum_{i=1}^{n_0} q(s_{0,i}, \pi) - \E_{\mu_0}[q(s_{0,i}, \pi)] \rv \leq   C_\Q^w \sqrt{\frac{2\log \frac{4|\Q|}{\delta}}{n_0}}
\end{align*}

Since $L_f^w(w,q) = \E_\eta[f_{s,a}(w(s,a))] + \E_{d^D}[l_i] + \E_{\init}[q(s_0, \pi)]$, but the first term can be calculated exactly, taking a union bound over the above two inequalities, we have that w.p. $\geq 1-\delta$, 
\begin{align*}
    |\wh{L}_f^w(w,q) - L_f^w(w,q)| &\leq   (1+\gamma)C_\W^w C_\Q^w  \sqrt\frac{2\log \frac{4|\Q||\W|}{\delta}}{n} + (1-\gamma)C_\Q^w \sqrt\frac{2\log \frac{4|\Q|}{\delta}}{n_0} 
\end{align*}
\end{proof}

\begin{lemma}\label{lem:w_estimation_error}
    Under Assumptions~\ref{assum:regularizer_w},~\ref{assum:realizable_w},~\ref{assum:bounded_w}, w.p. $\geq 1-\delta$, 
    \begin{align*}
        L_f^w(w_f^*, q_f^*) - L_f^w(\wh{w}, q_f^*) \leq 2 \epsilon_{stat}^w
    \end{align*}
\end{lemma}

\paragraph{Proof of Lemma~\ref{lem:w_estimation_error}}
Letting $\wh{q}(w) = \argmax_{q \in \Qcal} \wh{L}_f^w(w, q)$, we decompose the error as follows:
\begin{align*}
    L_f^w(\wh{w}, q_f^{*}) - L_f^w(w^\pi, q_f^{*}) &= L_f^w(\wh{w}, q_f^{*}) - \wh{L}_f(\wh{w},  q_f^{*}) &&\text{(1)  }\geq -\epsilon_{stat}^w \\
    &+ \wh{L}_f^w(\wh{w},  q_f^{*}) -  \wh{L}_f^w(\wh{w}, \wh{q}(\wh{w}))&&\text{(2)  } \geq 0 \\
    &+ \wh{L}_f^w(\wh{w}, \wh{q}(\wh{w})) - \wh{L}_f^w(w^\pi, \wh{q}(w^\pi))&&\text{(3)  }\geq 0 \\
    &+ \wh{L}_f^w(w^\pi, \wh{q}(w^\pi)) - L^w_f (w^\pi, \wh{q}(w^\pi)&&\text{(4)  }\geq -\epsilon_{stat}^w \\
    &+ L^w_f (w^\pi, \wh{q}(w^\pi)) - L^w_f (w^\pi, q_f^{*}) &&\text{(5)  }\geq 0
\end{align*}
Combining the inequalities gives the result. We give a brief justification for each term below. 
Terms (1) and (4) follow from Lemma~\ref{lem:w_statistical_error}. 

Term (2) $\geq 0$, since $q_f^{*} \in \Q$. 

Term (3) $\geq 0$ since $w^\pi \in \W$ and $\wh{w} = \argmax_{w \in \W} \wh{L}_f^w(w, \wh{q}(w))$. 

Term (5) $\geq 0$ since $(w^\pi, q_f^{*})$ is a saddle point solution.

\section{Additional Details of the Experiments} \label{appendix:experiments}
\subsection{Derivation}
We now derive the system of equations for our value function estimation experiments in Section~\ref{sec:experiments}. 
Letting the regularization function be $f_{s,a}(x) = \frac{1}{2}x^2$ for all $(s,a)$, the objective is
\begin{align}
    \min_{q} \max_{w} L_f^q(q,w) = \frac{1}{2}\E_\nu[q^2(s,a)] + \E_{d^D}\lb w(s,a) \lp r(s,a) + \gamma q(s',\pi) - q(s,a) \rp \rb,
\end{align} 
Letting $\E_n$ denote the empirical average over $\D$ for clarity, 
with empirical samples and the linear classes $\Q, \mathcal{W}$, the objective becomes:
\begin{align*}
    \min_{q \in \Q} \max_{w \in \W} \wh{L}_f^q(q,w) &= \frac{1}{2}\E_\nu[\alpha^\top \phi(s,a) \phi(s,a)^\top \alpha] + \beta^\top \Big( \E_n \lb \phi(s,a) r(s,a) \rb \\
    &\quad+ \E_n \lb \gamma \phi(s,a)\phi(s',\pi)^\top  - \phi(s,a)\phi(s,a)^\top \rb \alpha \Big)
\end{align*} 
Since $\beta \in \R^d$, $\max_{w \in \W} \wh{L}_f^q(q,w) = +\infty$ for any $q$, unless $\alpha$ sets the the second term to 0. This is satisfied by $\alpha$ such that 
\begin{align*}
    \E_n \lb \phi(s,a)\phi(s,a)^\top  - \gamma \phi(s,a)\phi(s',\pi)^\top \rb \alpha = \E_n \lb \phi(s,a) r(s,a) \rb. 
\end{align*}
However, there may in general be infinite feasible $\alpha$ depending on the linear features and samples. For our specific linear parameterization of $\Q, \W$, the constraints form an underdetermined $d \times k$ system of equations, which has infinite solutions. 

This is where the regularization term $\E_\nu[\alpha^\top \phi(s,a) \phi(s,a)^\top \alpha]$ comes into play. 
For any regularizing distribution $\nu$, our method will output a solution that minimizes this term, i.e. that minimizes the norm of $q = \Phi^\top \alpha$ on $\nu$. If $\nu = 0$, for example, the algorithm will output any feasible point; if $\nu = 1/|\Scal \Acal|$, the algorithm will output $q$ with smallest L2 norm. 
 
\paragraph{Connection to LSTDQ} When using the same linear class for $\Wcal$ and $\Qcal$, the solution to the constraints in Eq.\eqref{eq:primal_reg_q} (i.e., ignoring the regularization objective)---if the solution is unique given matrix invertibility---coincides with LSTDQ \citep{uehara2019minimax}. As mentioned in Section~\ref{sec:related}, LSTDQ enjoys function-estimation guarantees under matrix invertibility. In fact, we believe it is possible to extend the analysis even when $\Qcal$ and $\Wcal$ use different features of dimensions $d$ and $k$, respectively; as long as $k\ge d$ and the matrix in Eq.\eqref{eq:primal_reg_q} has full row-rank\footnote{In the finite-sample regime, one needs to lower-bound the smallest singular value of such matrices instead of imposing full-rankness \citep{perdomo2022sharp}.} (i.e., \textit{overdetermined}), similar guarantees for LSTDQ should still hold, though we are not aware of an explicit documentation of this fact. In contrast, our setup is more challenging as we are in the regime of $k < d$, and the constraints in Eq.\eqref{eq:primal_reg_q} is \textit{underdetermined}, nullifying the guarantees of LSTDQ. In such cases, the use of regularization is important for guaranteeing function estimation, as also shown in our experiments. 

\subsection{Experimental Setup}
\paragraph{Feature Design} In total, the tabular environment has 400 state-action values, and we design $\Phi$ to aggregate states that correspond to unique entries (within 3 decimal places) of $q^\pi$. 
In Figure 1, $\wt\Phi$ is composed of the set of features given by
$$\{(I - \gamma \Pt)^{-1} (\nu \emult q^\pi) / d^D, (I - \gamma \Pt)^{-1} (\nu \emult q^\pi)\}_{\nu \in \mathcal{V}}.$$ 
The first of these two entries is the closed-form solution of $w_f^*$ given in Lemma~\ref{lem:wfstar_closedform}, and satisfies the realizability requirements of all methods; the second is included for optimization stability.  

In Figure 2, we use a model with constant value equal to the average value of $q^\pi$ on the support of $p$, i.e. $\ol{q} = \nicefrac{1}{|\Scal\Acal|}\sum_{s,a} q^\pi(s,a) \cdot \mathbbm{1}_{\{p > 0\}}$. To maintain realizability when the model is included in the regularization function, $\wt\Phi$ is composed of the set 
$$\Big\{(I - \gamma \Pt)^{-1} (\nu \emult q^\pi), \;(I - \gamma \Pt)^{-1} (\nu \emult q^\pi \emult \mathbb{I}(\wt{q} > 0)), \;(I - \gamma \Pt)^{-1}( \nu \emult \mathbb{I}(\wt{q} > 0)) \Big\}_{\nu \in \mathcal{V}}$$ 
The reason why this preserves realizability is as follows. When $\nu$ is the regularization distribution, and the input model is $\wt{q} = (mq^\pi + (1-m)\ol{q}) \emult \mathbbm{1}(p > 0))$ for some constant $\ol{q}$, the closed-form solution $w_f^*$ can be expanded as 
\begin{align*}
    w_f^* &= (I - \gamma \Pt)^{-1} (\nu \emult (q^\pi - \wt{q})) \\
    &= (I - \gamma \Pt)^{-1}(\nu \emult q^\pi) - m \cdot (I - \gamma \Pt)^{-1} (\nu \emult \mathbbm{1}(p > 0) \emult q^\pi) \\
    &\quad- (1-m)\ol{q} \cdot (I - \gamma \Pt)^{-1} (\nu \emult \mathbbm{1}(p > 0)),
\end{align*}
which implies $w_f^*$ can be expressed as a linear combination of the three previously defined features. 

\paragraph{Solver} We solve the linear system using CVXPY with optimizer SCS~\citep{diamond2016cvxpy, agrawal2018rewriting}. 

\paragraph{Environment} The Gridwalk is a 10x10 environment with 4 actions corresponding to cardinal directions. The objective is to reach the goal state (lower right corner). In each state, the agent receives a reward inversely proportional to its distance from a goal state. Each trajectory terminates after 100 steps. The initial states are randomly distributed over the upper half of the grid. 

The target policy is defined to be a deterministic optimal policy that always moves towards the goal by first going right, and then down. To create a strong shift, the behavioral policy is designed to largely explore only the bottom left portion of the grid, providing poor coverage over the target policy and starting states. Specifically, letting the following probabilities refer to distributions over actions [RIGHT, DOWN, LEFT, UP], 
the target policy $\pi$ has distribution $[1, 0, 0, 0]$ over actions until it hits the right wall, then $[0, 1, 0, 0]$. The behavior policy takes $[0.1, 0.4, 0.5, 0]$ until it hits the right wall, then takes $[0, 0.5, 0.5, 0]$.

\section{Approximation and Optimization Error}\label{appendix:approximation}

The main results of this paper (Theorems~\ref{thm:q_bound},~\ref{thm:w_bound})
utilize assumptions on realizability (Assumption~\ref{assum:realizable_q},~\ref{assum:realizable_w}), as well as (implicit) assumptions of perfect optimization. 
In this section, we analyze how approximation errors, i.e. when the saddle point solution is not contained in $\Q \times \W$, and optimization errors affect our error bounds. 
Due to the similarity in proofs between value function and weight learning, 
we provide them only for value function learning; analogous methods can be used to derive similar results for weight learning. 

\subsection{Finite-sample Guarantees}
First, we relax the realizability requirements of Assumption~\ref{assum:realizable_q}. Define the approximation errors:
\begin{align*}
    \epsilon_{approx, q} &= \min_{q \in \Q} \max_{w \in \W} \lv \E_{d^D}[w(s,a)(\Tcal^\pi q(s,a) - q(s,a))] + \E_\nu[f_{s,a}(q(s,a)) - f_{s,a}(q^\pi(s,a))] \rv  \\
    \epsilon_{approx, w} &= \min_{w \in \W}\max_{q \in \Q} |\E_{d^D}[(w(s,a) - w_f^*(s,a))(\T^\pi q(s,a) - q(s,a))]| \\ 
    \epsilon_{approx} &:= \epsilon_{approx, q} + \epsilon_{approx, w}.
\end{align*}
$\epsilon_{approx,q}$ is composed of the worst-case weighted combination of Bellman errors of the best candidate $q \in \Q$, as well as the difference between the regularization function at $q$ and $q^\pi$.
The error $\epsilon_{approx,w}$ measures the distance between the best candidate $w \in \W$ and the saddle point solution $w_f^*$ by projecting the difference onto the worst-case Bellman error $\T^\pi q - q$. 
\begin{remark}
    To increase intuition of $\epsilon_{approx,q}$, we can relax the difference in regularization terms as $\E_\nu[f_{s,a}(q(s,a)) - f_{s,a}(q^\pi(s,a))] \leq C_{f'}^q ||q^\pi - q||_{2, \nu}$, which is also the norm upon which the $\wh{q}$ estimation guarantee is given (Theorem~\ref{thm:q_bound}) . 
    Reflecting the nature of the value function estimation task, this states that, even if there is a candidate $q \in \Q$ with low Bellman error (e.g. if data is sparse), $\epsilon_{approx,q}$ will still be large if $q$ is far from $q^\pi$ on the desired  distribution $\nu$. 
\end{remark}

Next, we can also relax the (implicit) assumptions that we obtain the true optima of \eqref{eq:sample_lagrangian_q}. Let $(\wh{q}, \wh{w})$ be the approximate solutions of \eqref{eq:sample_lagrangian_q} found by the algorithm. As before, define $\wh{w}(q) := \argmax_{w \in \W} \wh{L}(w,q)$ to be the true empirical maximizer for any $q \in \Qcal$. Note that since we allow for optimization error, it is not necessarily the case that $\wh{q} = \argmin_{q \in \Q}\wh{L}_f^q(q,\wh{w}(q))$ and $\wh{w} = \wh{w}(\wh{q}) = \argmax_{w \in \W} \wh{L}_f^q (\wh{q},w)$. Correspondingly, define the following optimization errors: 
\begin{align*}
    \epsilon_{opt, w} &\geq \wh{L}_f^q(\wh{q}, \wh{w}(\wh{q})) - \wh{L}_f^q(\wh{q},\wh{w}) \\ 
    \epsilon_{opt, q} &\geq \wh{L}_f^q(\wh{q},\wh{w}(\wh{q})) - \min_{q \in \Q} \wh{L}_f^q(q,\wh{w}(q)) \\ 
    \epsilon_{opt} &:= \epsilon_{opt,q} + \epsilon_{opt, w}. 
\end{align*}
$\epsilon_{opt,w}$ states that the estimate $\wh{w}$ should not be too far from the best discriminator in $\W$ for $\wh{q}$, while $\epsilon_{opt,q}$ states that the estimate $\wh{q}$ should not be too far from the minimax solution. 

Using the above definitions, we provide the following generalization of Theorem~\ref{thm:q_bound}, which accounts for approximation and optimization errors.  
\begin{theorem}\label{thm:approx_q_bound}
    Under Assumptions~\ref{assum:regularizer_q} and~\ref{assum:bounded_q}, with probability at least $1-\delta$, 
    \begin{align*}
        ||\wh{q} - q^\pi||_{2,\nu} \leq \sqrt{\frac{4\epsilon_{stat}^q + 2\epsilon_{approx} + 2\epsilon_{opt}}{M^q}},  
    \end{align*}
    where $\epsilon_{stat}^q$ is given in Theorem~\ref{thm:q_bound}. 
\end{theorem}

\subsection{Proof of Theorem~\ref{thm:approx_q_bound}}\label{proof:approx_q_bound}
The proof takes the same overall steps as the proof of Theorem~\ref{thm:q_bound} (Appendix~\ref{proof:q_bound}), but relies on Lemma~\ref{lem:approx_q_estimation_error} to incorporate the approximation and optimization errors: 
\begin{align*}
    ||\wh{q} - q^\pi||_{2,\nu} &\leq \sqrt{\frac{2\lp L_f^q(\wh{q},w_f^*) - L_f^q(q^\pi,w_f^*)\rp}{  M^q}} &&  \\
    &\leq \sqrt{\frac{4  \epsilon_{stat}^q + 2\epsilon_{approx,q} + 2\epsilon_{approx,w} + 2\epsilon_{opt,q} + 2\epsilon_{opt,w}}{  M^q}}.  && \text{(Lemma ~\ref{lem:approx_q_estimation_error})} 
\end{align*}

Below, we state and prove the helper lemma, which bounds the difference between the Lagrangian objective~\eqref{eq:lagrangian_reg_q} at the saddle point $(q^\pi,w_f^*)$ and the point $(\wh{q},w_f^*)$: 
\begin{lemma}\label{lem:approx_q_estimation_error}
    Under Assumptions~\ref{assum:regularizer_q} and~\ref{assum:bounded_q}, w.p. $\geq 1-\delta$, 
    \begin{align*}
        L_f^q(\wh{q},w_f^*) - L_f^q(q^\pi,w_f^*) \leq 2\epsilon_{stat}^q + \epsilon_{approx,q} + \epsilon_{approx,w} + \epsilon_{opt,q} + \epsilon_{opt,w}. 
    \end{align*}
\end{lemma}

\begin{proof}
    With some abuse of notation (as $\wt{q}, \wt{w}$ previously referred to models used with the regularizer), 
    for brevity in this section, let $\wt{q}$ be the minimizer of $\epsilon_{approx, q}$ and $\wt{w}$ be the minimizer of $\epsilon_{approx, w}$. That is, 
    \begin{align*}
        \wt{q} &= \argmin_{q \in \Q} \max_{w \in \W} \lv \E_{d^D}[w(s,a)(\Tcal^\pi q(s,a) - q(s,a))] + \E_\nu[f(q(s,a)) - f(q^\pi(s,a))] \rv  \\ 
        \wt{w} &= \argmin_{w \in \W}\max_{q \in \Q} |\E_{d^D}[(w(s,a) - w_f^*(s,a))(\T^\pi q(s,a) - q(s,a))]|.
    \end{align*}

    Decompose the error as follows: 
    \begin{align*}
        L_f^q(q^\pi, w_f^{*}) - L_f^q(\wh{q}, w_f^{*}) &=  L^q_f(q^\pi, w_f^{*}) -  L^q_f(q^\pi,\wh{w}(\wt{q})) && \text{(1)  } \geq 0\\
        &+ L^q_f(q^\pi, \wh{w}(\wt{q})) -  L^q_f(\wt{q}, \wh{w}(\wt{q})) && \text{(2)  }\geq -\epsilon_{approx,q} \\
        &+ L^q_f(\wt{q}, \wh{w}(\wt{q})) -  \wh{L}^q_f(\wt{q}, \wh{w}(\wt{q})) && \text{(3)  }\geq -\epsilon_{stat}\\
        &+ \wh{L}^q_f(\wt{q}, \wh{w}(\wt{q})) -  \wh{L}^q_f(\wh{q}, \wh{w}) && \text{(4)  }\geq -\epsilon_{opt, q} \\
        &+ \wh{L}^q_f(\wh{q}, \wh{w}) -  \wh{L}^q_f(\wh{q}, \wt{w}) && \text{(5)  }\geq -\epsilon_{opt, w} \\
        &+ \wh{L}^q_f(\wh{q}, \wt{w}) - L^q_f(\wh{q}, \wt{w}) && \text{(6)  }\geq -\epsilon_{stat}\\
        &+ L^q_f(\wh{q}, \wt{w})  - L^q_f(\wh{q}, w_f^*) && \text{(7)  }\geq -\epsilon_{approx, w} 
    \end{align*}
    
    First, (1) holds because $(q^\pi, w_f^*)$ is the saddle point solution of $L_f^q$ over all $q,w \in \R \times \R$. 
    The statistical errors in (3) and (6) follow from  Lemma~\ref{lem:q_statistical_error}. 
    
    Next, we justify the optimization errors. For (4), 
    \begin{align*}
        \wh{L}^q_f(\wt{q},\wh{w}(\wt{q})) - \wh{L}^q_f(\wh{q},\wh{w}) \geq \wh{L}^q_f(\wt{q},\wh{w}(\wt{q})) - \wh{L}^q_f(\wh{q},\wh{w}(\wh{q})) \geq \min_{q \in \Q} \wh{L}^q_f(q,\wh{w}(q)) - \wh{L}^q_f(\wh{q},\wh{w}(\wh{q})) \geq -\epsilon_{opt, q}. 
    \end{align*}
    For (5), 
    \begin{align*}
        \wh{L}^q_f(\wh{q},\wh{w}) -  \wh{L}^q_f(\wh{q},\wt{w}) \geq \wh{L}^q_f(\wh{q},\wh{w}) -  \max_{w \in \W} \wh{L}^q_f(\wh{q},w) \geq -\epsilon_{opt, w}
    \end{align*}
    
    Finally, we justify the approximation errors, starting with (2). Note that for any $q,w \in \Q\times\W$, 
    \begin{align*}
        |L_f^q(q^\pi,&w) - L_f^q(q,w)| \\
        &= |\E_{d^D}[w(s,a)(\Tcal^\pi q(s,a) - q(s,a) - \Tcal^\pi q^\pi(s,a) + q^\pi(s,a))] \\
        &\quad+ \E_\nu[f_{s,a}(q(s,a)) - f_{s,a}(q^\pi(s,a))]| \\ 
        &= |\E_{d^D}[w(s,a)(\Tcal^\pi q(s,a) - q(s,a))] +  \E_\nu[f_{s,a}(q(s,a)) - f_{s,a}(q^\pi(s,a))]| \\ 
        &\leq \max_{w \in \W}|\E_{d^D}[w(s,a)(\Tcal^\pi q(s,a) - q(s,a))] +  \E_\nu[f_{s,a}(q(s,a)) - f_{s,a}(q^\pi(s,a))]|.  
    \end{align*}
    Then since $\wt{q}$ was chosen to minimize the above expression, 
    \begin{align*}
        L^q_f(q^\pi,&\wh{w}(\wt{q})) - L^q_f(\wt{q},\wh{w}(\wt{q})) \\
        &\geq - \max_{w \in \W} |\E_{d^D}[w(s,a)(\Tcal^\pi \wt{q}(s,a) - \wt{q}(s,a))] +  \E_\nu[f_{s,a}(\wt{q}(s,a)) - f_{s,a}(q^\pi(s,a))]|\\
        &= -\min_{q \in \Q} \max_{w \in \W} |\E_{d^D}[w(s,a)(\Tcal^\pi q(s,a) - q(s,a))] +  \E_\nu[f_{s,a}(q(s,a)) - f_{s,a}(q^\pi(s,a))]| \\
        &= -\epsilon_{approx, q}.
    \end{align*}
    
    Next we justify (8). For any $w \in \W$ and $q \in \Q$, 
    \begin{align*}
        |L_f^q(q,w) - L_f^q(q,w_f^*)| &= |\E_{d^D}[(w(s,a) - w_f^*(s,a))(\T^\pi q(s,a) - q(s,a))]| \\ 
        &\leq \max_{q \in \Q} |\E_{d^D}[(w(s,a) - w_f^*(s,a))(\T^\pi q(s,a) - q(s,a))]|. 
    \end{align*}
    Then since $\wt{w}$ was chosen to minimize the RHS of the above inequality, 
    \begin{align*}
        L^q_f(\wh{q},\wt{w}) - L^q_f(\wh{q},w_f^*) &\geq -\max_{q \in \Q} |\E_{d^D}[(\wt{w}(s,a) - w_f^*(s,a))(\T^\pi q(s,a) - q(s,a))]| \\
        &= -\min_{w \in \W}\max_{q \in \Q}|\E_{d^D}[(w(s,a) - w_f^*(s,a))(\T^\pi q(s,a) - q(s,a))]| \\
        &= -\epsilon_{approx, w}. 
    \end{align*}
    Combining these inequalities gives the lemma statement. 
\end{proof}

\section{Off-Policy Return Estimation}\label{sec:ope} 
Section~\ref{sec:value} demonstrates how q-value estimates $\wh{q}$ can be obtained, and Section~\ref{sec:weight} demonstrates how weight estimates $\wh{w}$ can be obtained. The estimates $\wh{q}$ and/or $\wh{w}$ can additionally be used for downstream off-policy evaluation (OPE) of the policy's value $J(\pi)$, which can be equivalently defined in the following three ways:
\begin{align*}
    J(\pi) &= (1-\gamma)\E_{s_0 \sim \init}[q^\pi(s_0,\pi)] && \text{(``value function-based")}\\ 
    J(\pi) &= \E_{(s,a) \sim d^D, r \sim R(\cdot|s,a)}[w^\pi(s,a) \cdot r] && \text{(``weight-based")} \\ 
    J(\pi) &= (1-\gamma)\E_{s_0 \sim \init}[q^\pi(s_0,\pi)] && \text{(``doubly robust")} \\
    &\quad+ \E_{(s,a) \sim d^D, r \sim R(\cdot|s,a), s'\sim P(\cdot|s,a)}[w^\pi(s,a) ( r + q^\pi(s',\pi) - q^\pi(s,a))] 
\end{align*}
With finite samples and estimates $\wh{q}$ and $\wh{w}$ approximating $q^\pi$ and $w^\pi$, respectively, their corresponding off-policy estimators are: 
\begin{align*}
    \wh{J}^q(\pi) &= (1-\gamma) \frac{1}{n_0}\sum_{i=1}^{n_0} \wh{q}(s_{0,i},\pi) \\ 
    \wh{J}^w(\pi) &= \frac{1}{n}\sum_{i=1}^n \wh{w}(s_i,a_i)r_i \\ 
    \wh{J}^{\dr}(\pi) &= (1-\gamma) \frac{1}{n_0} \sum_{j=1}^{n_0} \wh{q}(s_{0,j},\pi) + \frac{1}{n}\sum_{i=1}^n \wh{w}(s_i,a_i)\lp r_i + \wh{q}(s_i',\pi) - \wh{q}(s_i,a_i)\rp 
\end{align*}
While the OPE estimator  $\wh{J}^{\dr}(\pi)$ utilizes both the weights and value functions, $\wh{J}^w(\pi)$ and $\wh{J}^q(\pi)$ utilize only one or the other. 
As a result, when $\wh{q}$ and $\wh{w}$ are estimated as in Sections~\ref{sec:value} and~\ref{sec:weight}, respectively,  $\wh{J}^w(\pi)$ and $\wh{J}^q(\pi)$ both inherit their $O(n^{-\nicefrac{1}{4}})$ sample complexities: 
\begin{corollary}\label{cor:q_evaluation}
    Suppose Assumptions~\ref{assum:regularizer_q},~\ref{assum:realizable_q}, and~\ref{assum:bounded_q} hold, 
    and let \newline$\wh{q} = \argmin_{q \in \Q} \max_{w \in \W} \wh{L}_f^q(q,w)$. Then with probability $\geq 1-2\delta$, 
    \begin{align*}
        |\wh{J}^q(\pi) - J(\pi)| \leq \epsilon_{eval}^q + \sqrt{\C_{\mu_0^\pi / \nu}} \cdot \epsilon^q_{est}, 
    \end{align*}
    where $\epsilon_{eval}^q = (1-\gamma)C_\Q^q\sqrt{\nicefrac{2\log \frac{2|\Q|}{\delta}}{n_0}}$,  $\C_{\mu_0^\pi/\nu} = || \nicefrac{\mu_0^\pi}{ \nu}||_\infty$, and $\epsilon_{stat}^q$ is as in Theorem~\ref{thm:q_bound}. 
\end{corollary}
\begin{corollary}\label{cor:w_evaluation}
    Suppose Assumption~\ref{assum:regularizer_w},~\ref{assum:realizable_w}, and~\ref{assum:bounded_w} hold, 
    and let \newline$\wh{w} = \argmin_{w \in \W} \max_{q \in \Q} \wh{L}_f^w(q,w)$. Then with probability $\geq 1-2\delta$, 
    \begin{align*}
        |\wh{J}^w(\pi) - J(\pi)| \leq \epsilon_{eval}^w +  \sqrt{\C_{d^D/\eta}} \cdot \epsilon_{est}^w, 
    \end{align*}
    where $\epsilon_{eval}^w = C_\W^w \sqrt{\frac{2\log \frac{2|\W|}{\delta}}{n}}$,  $\C_{d^D/\eta} = \|\nicefrac{d^D}{\eta}\|_\infty$, and $\epsilon_{est}^w$ is as in Theorem~\ref{thm:w_bound}.
\end{corollary}

However, when $\wh{q}$ and $\wh{w}$ are used together in the doubly robust estimator $\wh{J}^\dr$, their estimation error becomes multiplicative, and $\wh{J}^\dr(\pi)$ can achieve the $O(n^{-\frac{1}{2}})$ fast rate of convergence. In Theorem~\ref{thm:dr_evaluation} below, we present two versions this guarantee. The first requires no additional assumptions beyond $d^D > 0$, which we already make (see footnote \ref{footnote:data_coverage}), but involves the largest singular value of $I - \gamma P^\pi$, which may be difficult to characterize. The second utilizes an additional assumption, and replaces the singular value with an occupancy ratio, stated below. The assumption requires that all next states $s'$ are also present as states $s$ in transitions of $d^D$ (a condition which may reasonably hold in practice), and is also made by~\cite{uehara2021finite}. 
\begin{assumption}[Next State Coverage]\label{assum:sprime_coverage}
    Let $d^D(s) = \sum_a d^D(s,a)$ be the marginal distribution of states $s$ in $d^D$, and $d^D_{s'}(s) := \sum_{s',a'}P(s|s',a') d^D(s',a')$ be the marginal distribution of next states $s'$.  Suppose 
    \begin{align*}
        \C_{s'/s} &:= ||d^D_{s'}(\cdot) / d^D(\cdot)||_\infty < \infty
    \end{align*}
\end{assumption}
\begin{theorem}\label{thm:dr_evaluation}
    Suppose Assumption~\ref{assum:regularizer_q},~\ref{assum:realizable_q},~\ref{assum:bounded_q},~\ref{assum:regularizer_w},~\ref{assum:realizable_w}, and~\ref{assum:bounded_w} hold.
    Let $\wh{w}$ and $\wh{q}$ be estimated from: 
    \begin{align*}
        \wh{q} &= \argmin_{q \in \Q} \max_{w \in \W} \wh{L}_f^q(q,w)  \\ 
        \wh{w} &= \argmin_{w \in \W} \max_{q \in \Q} \wh{L}_f^w(q,w).  
    \end{align*}
    Then with probability $\geq 1 - 3\delta$, 
    \begin{align*}
        |\wh{J}^\dr(\pi) - J(\pi)| \leq \epsilon_{eval}^\dr + \sigma_{max}(I - \gamma P^\pi) \cdot \sqrt{\C_{d^D / \eta}\C_{d^D/\nu}} \cdot \epsilon^w_{est} \cdot \epsilon^q_{est},
    \end{align*}
    If Assumption~\ref{assum:sprime_coverage} additionally holds, with probability $\geq 1 - 3\delta$, 
    \begin{align*}
        |\wh{J}^\dr(\pi) - J(\pi)| \leq \epsilon_{eval}^\dr + \lp 1 + \gamma\sqrt{\C_{s'/s}\C_{\pi/\pi^D}}\rp \cdot  \sqrt{\C_{d^D / \eta}\C_{d^D/\nu}} \cdot \epsilon^w_{est}\ \cdot \epsilon^q_{est}, 
    \end{align*}
    where $\epsilon_{eval}^\dr = (1-\gamma)C_\Q^q\sqrt{\nicefrac{2\log \frac{2|\Q|}{\delta}}{n_0}} + C_\W^w(1 + (1+\gamma)C_\Q^q) \sqrt{\nicefrac{2\log \frac{2|\W||\Q|}{\delta}}{n}}$, $\sigma_{max}$ denotes the largest singular value, and $\epsilon^q_{est}$ and $\epsilon^w_{est}$ are as in Theorems~\ref{thm:q_bound} and~\ref{thm:w_bound}. 
\end{theorem}
As the evaluation error $\epsilon_{eval}^\dr$ in Theorem~\ref{thm:dr_evaluation} is $O(n^{-\nicefrac{1}{2}})$, the sample complexity of doubly robust estimation is rate-limited by $\epsilon_{est}^w \cdot \epsilon_{est}^q$, the product of weight and value function estimation errors. If both functions can be estimated at an $O(n^{-\nicefrac{1}{4}})$ rate, as is true of our method, then $\wh{J}^\dr(\pi)$ attains the overall $O(n^{-\nicefrac{1}{2}})$ fast rate. 
Finally, while Theorem~\ref{thm:dr_evaluation} assumes for simplicity that the same $\Q,\W$ classes are used in both of its optimization problems, it can easily be extended to the case where different pairs of function classes are used as long as the required assumptions hold.

\remark[Comparison to Related Work] \cite{yang2020off} conduct experiments comparing off-policy evaluation using $\wh{J}^q(\pi), \wh{J}^w(\pi),\wh{J}^\dr(\pi)$, and generally observe that $\wh{J}^\dr(\pi)$ has higher variance and worse performance than either $\wh{J}^q(\pi)$ or $\wh{J}^w(\pi)$. 
Though at first glance this may appear to contradict Theorem~\ref{thm:dr_evaluation}, that is actually not the case; in fact, our theoretical analysis provides insight into why \cite{yang2020off} may observe such a phenomenon. 
In contrast to Theorem~\ref{thm:dr_evaluation}, when using $\wh{J}^\dr(\pi)$ \cite{yang2020off} utilize saddle point predictions $(\wh{q}, \wh{w})$ from either \emph{only} value function learning or \emph{only} weight learning, e.g.  $(\wh{q}, \wh{w}) = \argmin_{q \in \Q} \argmax_{w \in \W} \wh{L}_f^q(q,w)$ that approximates $(q^\pi, w_f^*)$. 
Continuing with this example (and the same applies to weight learning), it is clear from our analysis that $\wh{w}$ estimated in such a manner may not approximate $w^\pi$ at all, leading to increased estimation error of $\wh{J}^\dr(\pi)$ over $\wh{J}^q(\pi)$. 
First, the closed-form solution we have derived for $w_f^*$ in (Lemma~\ref{lem:wfstar_closedform}) shows that $w_f^*$ may have a significantly different magnitude from $w^\pi$. 
Second, even if $\nu$ and $f$ were chosen such that $w_f^* \approx w^\pi$, as per the reasons stated in Section~\ref{sec:value_estimator}, we are not even guaranteed to output $\wh{w}$ close to $w_f^*$ since $L_f^q$ is not regularized in $w$. 
In order to obtain the estimation benefits of doubly robust estimation, 
our analysis shows that $\wh{q}$ and $\wh{w}$ should be separately estimated from their respective optimization problems, then combined in $\wh{J}^\dr(\pi)$. 
This is in accordance with similar results from \citet{kallus2020double} and \citet{uehara2021finite}.

\subsection{Proof of Corollary~\ref{cor:q_evaluation}} 
Let $\wt{J}(\pi) = (1-\gamma) \E_{\init}[\wh{q}(s,\pi)]$. We decompose the error as 
\begin{align*}
    |\wh{J}(\pi) - J(\pi)| \leq |\wh{J}(\pi) - \wt{J}(\pi)| + |\wt{J}(\pi) - J(\pi)|
\end{align*}
First we bound $|\wh{J}(\pi) - \wt{J}(\pi)|$. Using Hoeffding's with union bound, 
for all $q \in \Q$, w.p. $\geq 1-\delta$, 
\begin{align*}
    \lv \frac{1}{n_0} \sum_{i=1}^n q(s_{0,i}, \pi) - \E_{\mu_0}[q(s,\pi)] \rv \leq (1-\gamma)C_\Q^q \sqrt{\frac{2\log \frac{2|\Q|}{\delta}}{n_0}} := \epsilon_{eval}^q, 
\end{align*}
which implies $|\wh{J}(\pi) - \wt{J}(\pi)| \leq \epsilon_{eval}^q$. For the second term, let $C_{\mu_0^\pi/\nu} = ||\mu_0^\pi/\nu||_\infty$. Then w.p. $\geq 1-\delta$
\begin{align*}
    |\wt{J}(\pi) - J(\pi)| &= (1-\gamma)  |\langle \mu_0^\pi, \wh{q} - q^\pi \rangle| \\ 
    &\leq (1-\gamma) ||\wh{q} - q^\pi||_{1, \mu_0^\pi} \\
    &\leq (1-\gamma) ||\wh{q} - q^\pi||_{2, \mu_0^\pi} \\
    &= (1-\gamma) \sqrt{\C_{\mu_0^\pi/\nu}} ||\wh{q} - q^\pi||_{2, \nu} \\ 
    &\leq (1-\gamma) \sqrt{\C_{\mu_0^\pi/\nu}}\epsilon_{est}^q
\end{align*}
using Theorem~\ref{thm:q_bound} in the last line.

\subsection{Proof of Corollary~\ref{cor:w_evaluation}}\label{proof:ope} 
Let $\wt{J}(\pi) = \E_{d^D}[\wh{w}(s,a)r(s,a)]$. We decompose the error as 
\begin{align*}
    |\wh{J}^w(\pi) - J(\pi)| \leq |\wh{J}^w(\pi) - \wt{J}(\pi)| + |\wt{J}(\pi) - J(\pi)|
\end{align*}
For the first term, using Hoeffding's with union bound, w.p. $\geq 1-\delta$, for all $w \in \W$,  
\begin{align*}
    \lv \frac{1}{n} \sum_{i=1}^n w(s_i,a_i)r_i - \E_{d^D}[w(s,a)r(s,a)] \rv \leq C_{\W}^w \sqrt{\frac{2\log \frac{2|\W|}{\delta}}{n}} := \epsilon_{eval}^w
\end{align*}
which implies $|\wh{J}(\pi) - \wt{J}(\pi)| \leq \epsilon_{eval}^w$. For the second term,
\begin{align*}
    |\wh{J}(\pi) - J(\pi)| &= |\langle \wh{w} \cdot d^D, r \rangle - \langle w^\pi \cdot d^D, r \rangle| \\ 
    &\leq ||d^D \cdot (\wh{w} - w^\pi)||_1 ||r||_\infty \\ 
    &\leq ||d^D \cdot (\wh{w} - w^\pi)||_1  = ||\wh{w} - w^\pi||_{d^D, 1} \\
    &\leq ||\wh{w} - w^\pi||_{d^D,2} \\ 
    &\leq \sqrt{\C_{d^D/\eta}} ||\wh{w} - w^\pi||_{2, \eta}  \\
    &\leq \sqrt{\C_{d^D/\eta}}\epsilon_{est}^w
\end{align*}
w.p. $\geq 1-\delta$, using Theorem~\ref{thm:w_bound} in the last line. Taking a union bound over both terms gives the stated result.  

\subsection{Proof of Theorem~\ref{thm:dr_evaluation}}
Let $\wt{J}(\pi) = (1-\gamma) \E_{\mu_0^\pi}[\wh{q}(s,a)] + \E_{d^D}[\wh{w}(s,a)(r + \wh{q}(s',\pi) - \wh{q}(s,a))]$. Again we decompose the error as:  
\begin{align*}
    |\wh{J}^\dr(\pi) - J(\pi)| \leq |\wh{J}^\dr(\pi) - \wt{J}(\pi)| + |\wt{J}(\pi) - J(\pi)|. 
\end{align*}

For the first term, 
since $\E[\wh{J}^\dr(\pi)] = \wt{J}(\pi)$, w.p. $\geq 1-\delta$ we have that $\forall\; q,w \in \Q \times \W$, 
\begin{align*}
    |\wh{J}^\dr(\pi) - \wt{J}(\pi)| \leq (1-\gamma)C_\Q^q\sqrt{\frac{2\log \frac{2|\Q|}{\delta}}{n_0}} + C_\W^w(1 + (1+\gamma)C_\Q^q) \sqrt{\frac{2\log \frac{2|\W||\Q|}{\delta}}{n}} := \epsilon_{eval}^\dr
\end{align*}

For the second term, 
\begin{align*}
    |\wt{J}(\pi) - J(\pi)| &= |(1-\gamma)\langle \wh{q}, \mu_0^\pi\rangle + \langle \wh{w} \cdot d^D, r + \gamma P^\pi \wh{q} - \wh{q} \rangle - (1-\gamma)\langle q^\pi, \mu_0^\pi\rangle| \\
    &= |(1-\gamma)\langle \wh{q}, \mu_0^\pi\rangle + \langle \wh{w} \cdot d^D, r + \gamma P^\pi \wh{q} - \wh{q} \rangle - (1-\gamma)\langle q^\pi, \mu_0^\pi\rangle - \langle \wh{w} \cdot d^D, r + \gamma P^\pi q^\pi - q^\pi \rangle| \\ 
    &= | \langle \wh{q} - q^\pi, (1-\gamma) \mu_0^\pi + (\gamma P^{\pi,\top} - I)(d^D \cdot \wh{w})\rangle| \\
    &= \lv\llangle \wh{q} - q^\pi, (I - \gamma P^{\pi,\top})(d^D \cdot w^\pi - d^D \cdot \wh{w}) \rrangle\rv \\ 
    &\leq ||(I - \gamma P^\pi)(\wh{q} - q^\pi)||_{2,d^D}|| \wh{w} -  w^\pi||_{2,d^D} 
\end{align*}
where the last equality is due to the fact that $(1-\gamma) \mu_0^\pi = (I - \gamma P^\pi)(d^D \cdot w^\pi)$, and the final inequality is from Cauchy-Schwarz. We can automatically bound the $||\wh{w} - w^\pi||_{2,d^D}$ term using Theorem~\ref{thm:w_bound}, and it remains to bound $||(I - \gamma P^\pi)(\wh{q} - q^\pi)||_{2,d^D}$. We will consider two cases, first when $d^D > 0$ thus $\diag(d^D)$ is invertible,  and second, when  Assumption~\ref{assum:sprime_coverage} is satisfied. 

In the first case, let $D = \diag(d^D)$, which by assumption is invertible. Then 
\begin{align*}
    ||(I - \gamma P^\pi)(\wh{q} - q^\pi)||_{2,d^D}^2 &= (\wh{q} - q^\pi)^\top (I - \gamma P^\pi)^\top D (I - \gamma P^\pi)(\wh{q} - q^\pi) \\ 
    &= ||D^{1/2}(I - \gamma P^\pi)(\wh{q} - q^\pi)||_2^2 \\
    &= ||D^{1/2}(I - \gamma P^\pi)D^{-1/2}D^{1/2}(\wh{q} - q^\pi)||_2^2 \\
    &\leq ||D^{1/2}(\wh{q} - q^\pi)||_{2}^2 ||D^{1/2}(I - \gamma P^\pi)D^{-1/2}||_{2}^2 \\ 
    &= ||\wh{q} - q^\pi||_{d^D, 2}^2 ||I - \gamma P^\pi||_{2}^2 
\end{align*}
in the last line using the fact that the eigenvalues of a matrix $A$ and $L^{-1}AL$ are the same for any invertible matrix $L$. Thus, denoting the largest singular value of a matrix by $\sigma_{max}$, 
\begin{align*}
    |\wt{J}(\pi) - J(\pi)| &\leq \sigma_{max}\lp I - \gamma P^\pi\rp||\wh{w} - w^\pi||_{2,d^D}||\wh{q} - q^\pi||_{2, d^D}
\end{align*}
Using Theorem~\ref{thm:w_bound} and Theorem~\ref{thm:q_bound} in the last line to control the errors of $\wh{w}$ and $\wh{q}$ in the last line, followed by a union bound over the three inequalities, gives the result. 

For the second case, we can directly apply Lemma~\ref{lem:transition_bound}: 
\begin{align*}
    |\wt{J}(\pi) - J(\pi)| &\leq ||(I - \gamma P^\pi)(\wh{q} - q^\pi)||_{2,d^D}|| \wh{w} -  w^\pi||_{2,d^D} \\
    &\leq \lp ||\wh{q} - q^\pi||_{2,d^D} + \gamma ||P^\pi(\wh{q} - q^\pi)||_{2,d^D}\rp ||\wh{w} -  w^\pi||_{2,d^D} &&\\
    &\leq \lp 1+\gamma\sqrt{\C_{s'/s}\C_{\pi/\pi^D}}\rp||\wh{q} - q^\pi||_{2,d^D}|| \wh{w} -  w^\pi||_{2,d^D}, &&
\end{align*}
and again applying Theorem~\ref{thm:w_bound} and Theorem~\ref{thm:q_bound} gives the result. 

Lemma~\ref{lem:transition_bound} uses Assumption~\ref{assum:sprime_coverage} to bound the distance in value functions under the transition operator, and is stated and proved below. 
\begin{lemma}\label{lem:transition_bound}
    Under Assumption~\ref{assum:sprime_coverage}, 
    $$||P^\pi(\wh{q} - q^\pi)||_{2,d^D} \leq   \sqrt{\C_{s'/s}\C_{\pi/\pi^D}}||\wh{q} - q^\pi||_{2,d^D}. $$ 
\end{lemma}
\begin{proof}
    Define $||P^\pi||_{2,d^D} := \sup_{x \neq 0} || P^\pi x||_{2,d^D} / ||x||_{2,d^D}$. Then 
    \begin{align*}
        ||P^\pi(\wh{q} - q^\pi)||_{2,d^D} &\leq ||P^\pi||_{2,d^D} ||\wh{q} - q^\pi||_{2,d^D}. 
    \end{align*}
    It remains to bound $||P^\pi||_{2,d^D}$. For any $x$, 
    \begin{align*}
        ||P^\pi x||^2_{2,d^D} &= \E_{(s,a) \sim d^D}\lb \lp \E_{(s',a') \sim P^\pi(\cdot | s,a)}[x(s',a')]\rp^2 \rb \\
        &\leq \E_{(s,a,s',a') \sim d^D \times P^\pi }[x(s',a')^2] \\ 
        &\leq \max_{s,a} \lv \frac{d^D_{s'}(s)\pi(a|s)}{d^D(s)\pi^D(a|s)}\rv \E_{(s,a) \sim d^D}[x(s,a)^2] \\ 
        &= \C_{s'/s}\C_{\pi/\pi^D} ||x||_{2,d^D}^2
    \end{align*}
    This implies that $||P^\pi||_{2,d^D} \leq \sqrt{\C_{s'/s}\C_{\pi/\pi^D}}$, which gives the stated result. 
\end{proof}

\section{Infinite Function Classes}\label{appendix:infinite}
Our results for finite function classes can be easily extended to infinite function classes using covering numbers. We show that our method value function estimation under infinite function classes achieves the same $\wt{O}(n^{-1/4})$ rate as it does under finite function classes (Theorem~\ref{thm:q_bound}). The same results also apply to weight function learning using similar proof techniques. 

\subsection{Finite-sample Guarantees with Infinite Function Classes} 
First, we define the covering functions used in our results and analysis: 
\begin{definition}[Covering Number]\label{def:covering}
    For a function class $\F$, the covering number $\Ncal_{\infty}(\epsilon,\F)$ is defined to be the minimum cardinality of a set $\ol{\F} \subseteq \F$, such that for any $f \in \F$, there exists $\ol{f} \in \ol\F$ with $\|f - \ol{f}\|_\infty \leq \epsilon$.
\end{definition} 

Our guarantee for value function learning under infinite function classes is stated below, showing that we achieve the same rate as we do with finite classes. 
\begin{theorem}\label{thm:infinite_functions}
    Suppose Assumptions~\ref{assum:regularizer_q},~\ref{assum:realizable_q},~\ref{assum:bounded_q} hold. Then, with probability at least $1-\delta$, for $\epsilon = \frac{B}{2A\sqrt{n}}$, 
    \begin{align*}
        ||\wh{q} - q^\pi||_{2,\nu} \leq 2\sqrt{\frac{2B}{M^q}}\lp \frac{2\log \frac{2\Ncal_\infty(\epsilon,\Q)\Ncal_\infty(\epsilon,\W)}{\delta}}{n}\rp^{-\nicefrac{1}{4}},  
    \end{align*}
    where $\Ncal_\infty(\epsilon,\Q)$ and $\Ncal_\infty(\epsilon,\W)$ are as per Definition~\ref{def:covering}, and $A = 1 + (1+\gamma) C_{\Q}^q + 2(1+\gamma)C_{\W}^q$ and $B = C_\W^q(1+(1+\gamma)C_{\Q}^q)$. 
\end{theorem}
The proof is given below. 

\subsection{Proof of Theorem~\ref{thm:infinite_functions}} 
The statistical error of estimating $\wh{L}(q,w)$ under infinite function classes is the main technical detail of this proof. Given that, the stated bound on $\|\wh{q} - q^\pi\|_\nu$ can be derived using the same methods (leveraging strong convexity and Lemma~\ref{lem:q_estimation_error}) as were used in the proofs for value function estimation under finite function classes, i.e. for Theorem~\ref{thm:q_bound} (in Appendix~\ref{proof:q_bound}) and for Theorem~\ref{thm:approx_q_bound} (in Appendix~\ref{proof:approx_q_bound}). 

The bound on this statistical error is stated then proved below: 
\begin{lemma}[Statistical Error under Infinite Function Classes]\label{lem:stat_error_infinite}
    Suppose Assumption~\ref{assum:bounded_q} holds. Then, setting $\epsilon = \frac{B}{2A\sqrt{n}}$, for any $(q,w) \in \Q \times \W$ with probability at least $1-\delta$,
    \begin{align*}
        |L(q, w) - \wh{L}(q,w)| \leq 2B \sqrt{\frac{2\log \frac{2 \Ncal_\infty(\epsilon,\Q)\Ncal_\infty(\epsilon,\W)}{\delta}}{n}}, 
    \end{align*}
    where $\Ncal_\infty(\epsilon,\Q)$ and $\Ncal_\infty(\epsilon,\W)$ are as per Definition~\ref{def:covering}, and $A = 1 + (1+\gamma) C_{\Q}^q + 2(1+\gamma)C_{\W}^q$ and $B = C_\W^q(1+(1+\gamma)C_{\Q}^q)$. 
\end{lemma}

\begin{proof}[Proof of Lemma~\ref{lem:stat_error_infinite}]
    First, because the regularization term computes $\E_\nu[\cdot]$ exactly (not from samples), it has no effect on our bound. Formally, define the unregularized population Lagrangian to be
    $$L_0(q,w) = \E_{d^D}[r(s,a) + \gamma q(s',\pi) - q(s,a)],$$
    and its empirical version to be $\wh{L}_0(q,w)$. Then the LHS of Lemma~\ref{lem:stat_error_infinite} is equivalent to 
    \begin{align*}
        |L(q,w) - \wh{L}(q,w)| &= |L_0(q,w) + \E_\nu[f_{s,a}(q(s,a))] - \wh{L}_0(q,w) - \E_\nu[f_{s,a}(q(s,a))]| \\
        &= |L_0(q,w) - \wh{L}_0(q,w)|,
    \end{align*}
    so it suffices to bound the statistical error of estimating the unregularized Lagrangian $\wh{L}_0$. 
    
    For some (later to-be-specified) $\epsilon > 0$, let $\ol\Q$ be a minimal $\epsilon$-covering of $\Q$ in the infinity norm as per Definition~\ref{def:covering}, that is, $|\ol\Q| = \Ncal_\infty(\epsilon,\Q)$. Let $\ol\W$ be defined similarly for $\W$. 
    Then for any $(q,w) \in \Q \times \W$, let $(\ol{q}, \ol{w}) \in \ol\Q \times \ol\W$ be such that $\|q - \ol{q}\|_\infty \leq \epsilon$ and $\|w - \ol{w}\|_\infty \leq \epsilon$. By triangle inequality, 
    \begin{align*}
        |L_0(q,w) - \wh{L}_0(q,w)| &\leq |L_0(q,w) - \wh{L}_0(q,w) - L(\ol{q},\ol{w}) - \wh{L}(\ol{q}, \ol{w})| + |L(\ol{q}, \ol{w}) - \wh{L}(\ol{q}, \ol{w})|
    \end{align*}
    Next, define $\ell_{sas'}(q,w) := w(s,a)(r(s,a) + \gamma q(s',\pi) - q(s,a))$ such that $L_0(q,w) = \E_{d^D}[\ell_{sas'}(q,w)]$ and $\wh{L}_0(q,w) = \frac{1}{n} \sum_{i=1}^n \ell_{s_i a_i s_i'}(q,w)$. Then we can further upper bound the above as: 
    \begin{align*}
        |L_0(q,w) - \wh{L}_0(q,w)| &\leq 2\max_{s,a,s'} \underbrace{|\ell_{sas'}(q,w) - \ell_{sas'}(\ol{q}, \ol{w})|}_{\text{(T1)}} + + \underbrace{|L(\ol{q}, \ol{w}) - \wh{L}(\ol{q}, \ol{w})|}_{\text{(T2)}}.
    \end{align*}
    Term (T1) can be controlled using the $\epsilon$-covering definition, and (T2) can be controlled using standard concentration methods. Their respective bounds are provided below, with proofs in the next subsection: 
    \begin{lemma}[Bound for T1]\label{lem:t1_bound}
        Let $\ol{\Q}$ and $\ol{\W}$ be $\epsilon$-coverings of $\Q$ and $\W$, respectively, satisfying Definition~\ref{def:covering}. Then for any $(q,w) \in \Q \times \W$, there exists $(\ol{q},\ol{w}) \in \ol\Q \times \ol\W$ such that $\|q - \ol{q}\|_\infty \leq \epsilon$ and $\|w - \ol{w}\|_\infty \leq \epsilon$, and if Assumption~\ref{assum:bounded_q} holds,
        \begin{align*}
            \max_{sas'}|\ell_{sas'}(q,w) - \ell_{sas'}(\ol{q},\ol{w})| \leq A\epsilon, 
        \end{align*}
        with $A =1 + (1+\gamma) C_{\Q}^q + 2(1+\gamma)C_{\W}^q$. 
    \end{lemma}
    \begin{lemma}[Bound for T2]\label{lem:t2_bound}
        Let $\ol{\Q}$ and $\ol{\W}$ be minimal $\epsilon$-coverings of $\Q$ and $\W$, respectively, as in Definition~\ref{def:covering}, that is, $|\ol{\Q}| = \Ncal_\infty(\epsilon,\Q)$ and $|\ol{\W}| = \Ncal_\infty(\epsilon,\W)$. Then if Assumption~\ref{assum:bounded_q} holds, for any $(\ol{q}, \ol{w}) \in \ol{\Qcal} \times \ol{\Wcal}$ w.p. $\geq 1-\delta$, 
        \begin{align*}
            |L_0(\ol{q}, \ol{w}) - \wh{L}_0(\ol{q}, \ol{w})| \leq B\sqrt{\frac{2\log \frac{\Ncal_\infty(\epsilon,\Q)\Ncal_\infty(\epsilon,\W)}{\delta}}{n}},  
        \end{align*}
        where $B = \C_{\W}^q(1 + (1+\gamma) \C_{\Q}^q)$. 
    \end{lemma}
    Putting these two bounds together, letting $A$ be as in Lemma~\ref{lem:t1_bound} and $B$ be as in Lemma~\ref{lem:t2_bound}, we have 
    \begin{align*}
        |L_0(q,w) - \wh{L}_0(q,w)| &\leq 2 A\epsilon + B\sqrt{\frac{\log \frac{2 \Ncal_\infty(\epsilon,\Q)\Ncal_\infty(\epsilon,\W)}{\delta}}{n}}. 
    \end{align*}
    
    Choosing $\epsilon = \frac{B}{2A\sqrt{n}}$ gives the final bound:
    \begin{align*}
        |L(q,w) - \wh{L}(q,w)| = |L_0(q,w) - \wh{L}_0(q,w)|
        &\leq \frac{B}{\sqrt{n}} + B\sqrt{\frac{\log \frac{2 \Ncal_\infty(\epsilon,\Q)\Ncal_\infty(\epsilon,\W)}{\delta}}{n}} \\ 
        &\leq 2B\sqrt{\frac{\log \frac{2 \Ncal_\infty(\epsilon,\Q)\Ncal_\infty(\epsilon,\W)}{\delta}}{n}}.
    \end{align*}
\end{proof}

\subsection{Proofs for Helper Lemmas}
The proofs of Lemmas~\ref{lem:t1_bound} and~\ref{lem:t2_bound} are given below:
\begin{proof}[Proof of Lemma~\ref{lem:t1_bound}]
    For any $s,a,s'$, (since this tuple is fixed, going forward, we drop the $s,a,s'$ subscript from $\ell$ for brevity)
    \begin{align*}
        |\ell_{sas'}(q,w) - \ell_{sas'}(\ol{q},\ol{w})| &= |\ell(q,w) - \ell(q,\ol{w}) + \ell(q,\ol{w}) - \ell(\ol{q},\ol{w})| \\
        &\leq \underbrace{|\ell(q,w) - \ell(q,\ol{w})|}_{\text{T3}} + \underbrace{|\ell(q,\ol{w}) - \ell(\ol{q},\ol{w})|}_{\text{T4}}
    \end{align*}
    (T3) expresses the error from the covering approximation for $w$, while (T4) expresses this for $q$. First, to bound (T3), 
    \begin{align*}
        |\ell(q,w) - \ell(q,\ol{w})| &= |(w(s,a) - \ol{w}(s,a))(r(s,a) + \gamma q(s',\pi) - q(s,a)| \\ 
        &\leq \|w - \ol{w}\|_\infty(\|r\|_\infty + (1+\gamma)\|q\|_\infty) \\ 
        &\leq \epsilon(1 + (1+\gamma) C_{\Q}^q). 
    \end{align*}
    To bound (T4), 
    \begin{align*}
        |\ell(q,\ol{w}) - \ell(\ol{q},\ol{w})| &= |\ol{w}(s,a)(\gamma q(s',\pi) - q(s,a) - \gamma \ol{q}(s',\pi) + \ol{q}(s,a))| \\ 
        &\leq 2(1+\gamma) \|\ol{w}\|_\infty \| q - \ol{q}\|_\infty \\
        &\leq 2(1+\gamma) C_{\W}^q \epsilon. 
    \end{align*}
    Since these two inequalities hold for any $sas'$, combining them directly gives lemma statement. 
\end{proof}
\begin{proof}[Proof of Lemma~\ref{lem:t2_bound}]
    This is a straightforward application of Hoeffding's with union bound over $\ol{\Q},\ol{\W}$, akin to the proof of Lemma~\ref{lem:q_statistical_error} (which is over $\Q, \W$). 
\end{proof}

\end{document}